\pdfoutput=1
\PassOptionsToPackage{table}{xcolor}

\documentclass{article} %
\usepackage{preprint,times}
\prepfinalcopy

\usepackage[utf8]{inputenc} %
\usepackage[T1]{fontenc}    %
\usepackage{url}            %
\usepackage{booktabs}       %
\usepackage{amsfonts}       %
\usepackage{nicefrac}       %
\usepackage{microtype}      %
\usepackage[table]{xcolor}  %
\usepackage{multirow}
\usepackage{graphicx}  %

\usepackage{subcaption}

\usepackage{algorithm}
\usepackage[noend]{algorithmic}
\usepackage{amssymb}
\usepackage{mathtools}
\usepackage{amsmath}
\usepackage{amsthm}

\theoremstyle{plain}
\newtheorem{theorem}{Theorem}[section]
\newtheorem{proposition}[theorem]{Proposition}
\newtheorem{lemma}[theorem]{Lemma}
\newtheorem{corollary}[theorem]{Corollary}
\theoremstyle{definition}
\newtheorem{definition}[theorem]{Definition}
\newtheorem{assumption}[theorem]{Assumption}
\newtheorem{example}[theorem]{Example}
\theoremstyle{remark}

\DeclareMathOperator*{\argmin}{arg\,min}
\usepackage{wrapfig}
\usepackage{float}
\usepackage{thmtools}   %

\newcommand{\mn}{\mathcal{M}_n}
\newcommand{\ma}{\mathcal{M}_a}

\newcommand{\R}{\mathbf{R}}
\newcommand{\rc}{\texttt{minrc}}
\newcommand{\prc}{\widehat{\mathcal{R}}}
\let\hat\widehat

\newcommand{\Pn}{P_{n}}            %
\newcommand{\Pa}{P_{a}}            %
\newcommand{\Gtrue}{G^{*}}         %
\newcommand{\Rtrue}{R^{*}}         %
\newcommand{\M}{\mathcal{M}}       %
\newcommand{\PossRC}{\texttt{SolRC}} %
\newcommand{\PossCn}{\texttt{R}} %
\newcommand{\Yv}{Y} 
\newcommand{\Rv}{V} 
\newcommand{\V}{\mathbf{V}}       %

\newcommand{\indep}{\perp\!\!\!\perp}

\usepackage{xspace}
\newcommand{\myalgo}{\texttt{RCA-DCM}\xspace}
\usepackage{hyperref} %
\usepackage{cleveref}  %

\usepackage{tikz}
\usetikzlibrary{positioning,fit}
\usetikzlibrary{calc}
\usetikzlibrary{arrows.meta, shapes.geometric}

\definecolor{Fclr}{RGB}{160,30,30}        %
\definecolor{shiftbase}{HTML}{2F4B7C}     %
\colorlet{shiftT}{shiftbase!90!white}     %
\colorlet{shiftD}{shiftbase!60!white}
\colorlet{shiftY}{shiftbase!32!white}
\colorlet{shiftC}{shiftbase!12!white}     %
\tikzset{
  >={Stealth[length=1.4mm, width=1.0mm]},
  cnode/.style = {circle, draw=black, line width=0.4pt, fill=white,
                  minimum size=5mm, inner sep=0pt, font=\scriptsize},
  fnode/.style = {rectangle, draw=Fclr, line width=0.5pt, fill=Fclr!10,
                  text=Fclr, minimum size=4.2mm, inner sep=0pt, font=\scriptsize},
  dedge/.style = {->, line width=0.4pt, black},
  bidir/.style = {<->, line width=0.4pt, dashed, black},
  fedge/.style = {->, line width=0.5pt, Fclr},
  fedgeq/.style= {->, line width=0.5pt, dashed, Fclr},   %
  sT/.style = {cnode, fill=shiftT, text=white},
  sD/.style = {cnode, fill=shiftD, text=white},
  sY/.style = {cnode, fill=shiftY},
  sC/.style = {cnode, fill=shiftC},
}

\tikzset{
  directed/.style={->, thick},
  bidirected/.style={<->},
  conf/.style={bidirected, dashed}, %
  nnvsm/.style={
    rectangle, draw, very thick, fill=gray!28, inner sep=0.0cm, minimum width=1.0cm, minimum height=1.0cm, rounded corners=0.05cm
  }
}
\title{Beyond Conditional Independence: Root Cause Analysis with Deep Causal Models}

\author{%
  Md Musfiqur Rahman$^{1,*}$ \quad
  Kenneth Lee$^{1}$ \quad
  Ziwei Jiang$^{2}$ \quad
  Padmaja Jonnalagedda$^{3}$ \\
  \textbf{Ruocheng Guo$^{4,\dagger}$ \quad
  Murat Kocaoglu$^{2}$} \\
  {\small $^{1}$Electrical and Computer Engineering, Purdue University} \\
  {\small $^{2}$Computer Science, Johns Hopkins University} \\
  {\small $^{3}$Intuit AI Research} \\
  {\small $^{4}$Microsoft}
}

\begin{document}

\maketitle
{\renewcommand{\thefootnote}{\fnsymbol{footnote}}\footnotetext[1]{Corresponding author: \texttt{rahman89@purdue.edu}}\footnotetext[2]{Work completed while at Intuit AI Research.}}
\lhead{Preprint. Under review.}

\begin{abstract}
Root cause analysis (RCA) is a critical problem in many real-world scenarios. RCA enables the identification of faulty or failing mechanisms in a system by comparing anomalous observations with corresponding reference (i.e., regular) observations. However, existing approaches rely either on heuristic methods or on conditional independence tests with a strong unconfoundedness assumption, and thus fail to exploit other complicated distributional constraints in the presence of latent variables.
To relax these assumptions, we model the underlying system as a causal model and the anomalous system as a change in the structural functions of the same causal model. Specifically, to handle unobserved confounders, we establish an implicit connection between distributional constraint testing and root cause analysis. To adapt our approach to data generated from arbitrary causal models, we employ the deep causal model (DCM) framework, in which we design the causal model using neural networks. Finally, we illustrate how our method, RCA-DCM, can utilize different levels of partial graphical knowledge to perform RCA.
We evaluate RCA-DCM against state-of-the-art baselines on simulated datasets,  a physics-based causal chamber and two microservice applications. RCA-DCM improves top-1 accuracy over the strongest baseline on both Sock Shop (\textbf{0.880} vs.\ 0.752) and Online Boutique (\textbf{0.776} vs.\ 0.712), and when the true root cause in the causal chamber is unobserved and acts as a latent confounder, it recovers the exact root-cause set more often than any competing method (PRR \textbf{0.846} vs.\ 0.731).

\end{abstract}

\section{Introduction}

Root cause analysis (RCA) aims to identify the components of a complex system
that are responsible for anomalies. This is an important problem in many
domains, including microservice systems and cloud
applications~\citep{wang2023hierarchical,lin2024root,liu2021microhecl,ikram2022root,shan2019epsilon,ma2020automap},
economics~\citep{bai2024causal,inoue2021new}, scientific experiments, health
monitoring~\citep{strobl2023identifying,strobl2024counterfactual}, fraud
detection~\citep{vanhoeyveld2020vat}, and credit
scoring~\citep{das2023algorithmic}. Recent critical AI applications, such as
failure attribution in LLM-based multi-agent
systems~\citep{zhang2025agent}, can also be modeled as RCA problems.

The RCA problem is challenging because failures or abnormal behavior can
propagate through the system's structural dependencies. Many existing RCA
methods rely on heuristic rules, statistical measures, or correlation-based
analyses~\citep{pham2024baro,knorr1999finding,liu2017contextual,micenkova2013explaining,macha2018explaining,gupta2019beyond}.
These approaches may fail to identify the root cause effectively,
particularly in systems with complex structures, where a change in a
mechanism affects the distributions of downstream variables in nontrivial
ways. Thus, causal knowledge is essential for distinguishing the root cause
from other variables.

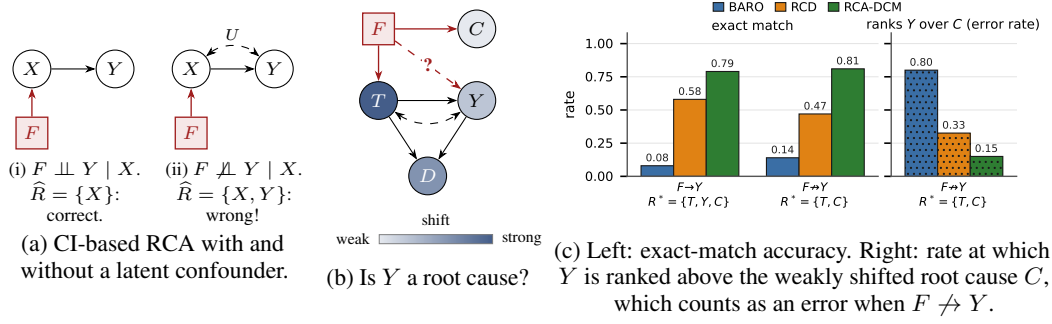
\begin{figure}[t!]
\vspace{-9mm}
\centering
\begin{minipage}[c]{0.29\linewidth}
\centering
\begin{minipage}[t]{0.48\linewidth}
\centering
\begin{tikzpicture}
  \node[cnode] (X) at (0,0)     {$X$};
  \node[cnode] (Y) at (1.1,0)   {$Y$};
  \node[fnode] (F) at (0,-0.85) {$F$};
  \path (0.55,0.45) node {};                 %
  \draw[fedge] (F) -- (X);
  \draw[dedge] (X) -- (Y);
\end{tikzpicture}\\[0.5mm]
{\scriptsize (i) $F \indep Y \mid X$.\\ $\hat{R}=\{X\}$: correct.\par}
\end{minipage}\hfill
\begin{minipage}[t]{0.48\linewidth}
\centering
\begin{tikzpicture}
  \node[cnode] (X) at (0,0)     {$X$};
  \node[cnode] (Y) at (1.1,0)   {$Y$};
  \node[fnode] (F) at (0,-0.85) {$F$};
  \draw[fedge] (F) -- (X);
  \draw[dedge] (X) -- (Y);
  \draw[bidir] (X) to[bend left=40]
        node[midway, above=-1pt, font=\tiny] {$U$} (Y);
\end{tikzpicture}\\[0.5mm]
{\scriptsize (ii) $F \not\indep Y \mid X$.\\ $\hat{R}=\{X,Y\}$: wrong!\par}
\end{minipage}\\[1mm]
{\footnotesize (a) CI-based RCA with and without a latent confounder.\par}
\end{minipage}
\hfill
\begin{minipage}[c]{0.19\linewidth}
\centering
\begin{tikzpicture}
  \node[fnode] (F) at (0,0.95)    {$F$};
  \node[sC]    (C) at (1.3,0.95)  {$C$};
  \node[sT]    (T) at (0,0)       {$T$};
  \node[sY]    (Y) at (1.3,0)     {$Y$};
  \node[sD]    (D) at (0.65,-1.0) {$D$};
  \draw[fedge]  (F) -- (T);
  \draw[fedge]  (F) -- (C);
  \draw[fedgeq] (F) -- node[midway, fill=white, inner sep=0.6pt,
                            font=\scriptsize\bfseries, text=Fclr] {?} (Y);
  \draw[dedge] (T) -- (Y);
  \draw[dedge] (T) -- (D);
  \draw[dedge] (Y) -- (D);
  \draw[bidir] (T) to[bend right=35] (Y);
\end{tikzpicture}\\[1mm]
\begin{tikzpicture}
  \shade[left color=shiftC, right color=shiftT, draw=black!50, line width=0.3pt]
    (0,0) rectangle (1.5,0.12);
  \node[font=\tiny, anchor=east]  at (0,0.06)    {weak};
  \node[font=\tiny, anchor=west]  at (1.5,0.06)  {strong};
  \node[font=\tiny, anchor=south] at (0.75,0.12) {shift};
\end{tikzpicture}\\[0.5mm]
{\footnotesize (b) Is $Y$ a root cause?\par}
\end{minipage}
\hfill
\begin{minipage}[c]{0.48\linewidth}
\centering
\includegraphics[width=\linewidth]{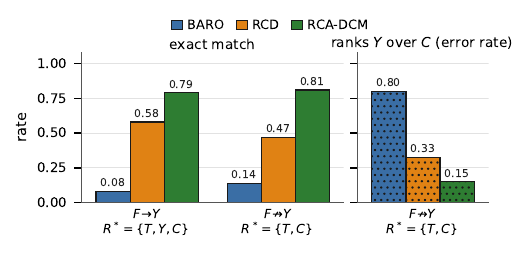}\\[1mm]
{\footnotesize (c) Left: exact-match accuracy. Right: rate at which $Y$ is
  ranked above the weakly shifted root cause $C$, which counts as an error
  when $F\not\to Y$.\par}
\end{minipage}
\caption{Baselines fail when latent confounders are present and
  non-root-cause variables experience larger distributional shifts than the
  root causes; Given $F\to T$, deciding
  whether $Y$ is a root cause ($F \to Y$) is equivalent to deciding whether
  $F$ is a valid instrument for the effect of $T$ on $Y$.}
\label{fig:profit-motiv}
\vspace{-3mm}
\end{figure}

When observations or metrics of the system are available both during normal operation (the normal dataset) and after a failure
occurs (the anomalous dataset),
causal discovery provides a framework for RCA by modeling anomalies, i.e.,
shifts in the observed distribution, as the result of a change in the
mechanism of the root cause, also known as a soft
intervention~\citep{jaber2020causal,okati2024root}. These approaches add a
binary indicator variable $F$ to the causal graph, with outgoing edges to the
intervention targets, to represent their normal mechanisms ($F=0$) and their
post-intervention mechanisms ($F=1$). Recent works such
as~\citet{ikram2022root,ikram2025root} treat the change from $F=0$ to $F=1$ as
the occurrence of an anomaly and perform causal discovery to find the
children of $F$, which they report as the anomalous nodes.
However, these methods rely solely on distributional invariances of the form $p_n(x \mid s)=p_a(x \mid s)$ for subsets $s$ and can fail to correctly rank the root causes in the presence of unobserved confounders, as in the graph in Figure~\ref{fig:profit-motiv}(a). Here, the true root cause is $R^*=\{X\}$, indicated by the edge $F\rightarrow X$. Since $F \not\indep X$, predicting $X$ as an RC is correct. However, in Fig~\ref{fig:profit-motiv}(a.ii), $F \not\indep Y \mid X$ due to the backdoor path $F\rightarrow X \leftrightarrow Y$, so using this dependency to declare $Y$ a root cause would be an incorrect prediction. Meanwhile, many non-causal methods based on statistical
testing~\citep{pham2024baro,li2022circa} assume homogeneous anomalies, where
the true root causes exhibit the largest distributional shift among all
variables, and therefore primarily measure marginal shifts. In the real
world, however, a shift originating at a root cause $R_1$ propagates to its
descendants, and a downstream non-root-cause descendant can end up with a
larger observed marginal shift than another root cause $R_2$.

Figure~\ref{fig:profit-motiv}(b) shows a scenario, averaged over 100 SCMs, in which both families of approaches fail. We consider a large marginal shift (anomaly) on $T$ and a mild anomaly on $C$, with $Y$ either a root cause (case 1, $F\rightarrow Y$) or not (case 2, $F\not\rightarrow Y$). In both cases, the non-root-cause $D$ inherits a larger marginal shift from its parent $T$ than the mildly anomalous $C$ exhibits, so statistical testing approaches such as BARO, which rank variables by marginal shift, rarely recover $R^*$ (Figure~\ref{fig:profit-motiv}(c), left). Causal approaches that assume unconfoundedness, such as RCD, fail for a different reason: $F\not\indep Y \mid T$ in both cases, due to the direct edge $F\rightarrow Y$ in case 1 and the path $F\rightarrow T \leftrightarrow Y$, opened by conditioning on $T$, in case 2. RCD thus cannot tell the two cases apart and often ranks the non-root-cause $Y$ above the mild root cause $C$ (Figure~\ref{fig:profit-motiv}(c), right). Moreover, conditional independence tests such as $D\indep F \mid T,Y$ become unreliable with growing conditioning sets. Hence, neither marginal shifts nor conditional independencies suffice for RCA under unobserved confounding; RCA in this setting requires exploiting other distributional constraints, which no existing work offers.
In this paper, we make a key observation: once $T$ has been identified as a
root cause ($F \rightarrow T$), asking whether $Y$ is also a root cause
($T \leftarrow F \rightarrow Y$) in the presence of a confounder is \emph{equivalent
to asking whether $F$ is an invalid instrument for the effect of $T$ on $Y$}.
Consequently, if the observed distribution is incompatible with $F$ being a
valid instrument, then both $T$ and $Y$ must be root causes. Unlike
conditional independence tests, which can be misleading in this setting, the
instrument constraint determines when $Y$ should be declared a root cause and
ranked above other variables.

It is natural to ask which other distributional
constraints~\cite{verma1990equivalence, ansanelli2025observational} can help distinguish root causes when these fail. However,
verifying such constraints, including the instrument constraint, is
nontrivial for arbitrary graphs with high-dimensional data. To address these
limitations, we propose a new causal framework that models RCA as a form of
unknown intervention target detection and leverages deep causal generative
models to reason about which mechanism shifts are necessary to explain the
observed distributional changes. This DCM-based RCA framework achieves $79\%$ and
$81\%$ exact-match accuracy in the setup of Figure~\ref{fig:profit-motiv}.
Our contributions are:
\begin{itemize}
      \item We propose \myalgo, an algorithm for root cause analysis in the
    presence of an arbitrary number of unobserved confounders, which uses a
    deep causal model to search for SCMs consistent with both the normal and
    anomalous datasets and identifies as root causes the variables whose
    mechanisms must shift to explain the anomaly.
\item We characterize the non-identifiable case, in which the root cause
    cannot be uniquely determined as another variable set can explain
    the distributional shift. We show how the set of root cause solutions
    changes when the given graph is mis-specified.
    \item We also demonstrate that \myalgo outperforms state-of-the-art baselines in almost all cases when evaluated on synthetic setups with unobserved confounders, a real-world physics-based testbed
    (Causal Chamber), and microservice datasets (Sock Shop and Online Boutique).
\end{itemize}

\section{Related Work}\label{sec:related}
Causal inference has become a principled approach to RCA, particularly in microservice systems. Early methods construct a causal graph over service-level metrics, often using constraint-based discovery such as PC, and localize root causes by traversing or ranking nodes with anomaly scores \citep{chen2014causeinfer, wang2018cloudranger, ma2019ms, meng2020localizing}. More recent works model failures as interventions on causal mechanisms: RCD~\citep{ikram2022root} and RCG~\citep{ikram2025root} treat failures as soft interventions and use local causal discovery or partial structural knowledge, CIRCA~\citep{li2022circa} casts RCA as intervention recognition in a causal Bayesian network, and CausalRCA~\citep{xin2023causalrca} learns the graph via gradient-based causal discovery. Other methods relax structural requirements, either with guarantees for restricted graph classes under missing structural knowledge~\citep{orchard2025missingstructural} or by exploiting linear-SCM invariances~\citep{li2024cholesky}, while statistical approaches such as BARO~\citep{pham2024baro} score distributional shifts without using causal structure. However, these methods either ignore causal structure, rely on conditional independence or invariance tests, or depend on linearity or graphs that may be inaccurate under hidden confounding; none explicitly models nonlinear mechanisms with latent confounders. In contrast, \myalgo uses deep causal models to exploit general distributional constraints under an arbitrary number of unobserved confounders. We provide an extended discussion in Appendix~\ref{app:related}.
\section{Background}
\begin{definition}[ SCM~\citep{pearl2009causality}]
An SCM $\mathcal{M}$ is a $4$-tuple 
{$ \mathcal{M}\!=\!(\mathcal{V},  \mathcal{U}, \mathcal{F}, P(.) )$,} where each observed variable $V_i\in\mathcal{V}$ is realized as an evaluation of a function $f_i^*\in\mathcal{F}$ that looks at a subset of the remaining observed variables $Pa_i\!\subset\! \mathcal{V}$,  and an unobserved (latent) variable $U_i\!\in\!\mathcal{U}$. 
This refers to the semi-Markovian causal model.
$P$ is a product joint distribution over all unobserved variables $\mathcal{U}$. 
\end{definition}

\begin{definition}[Graph with Latents, ADMG]
Each SCM induces an acyclic directed graph called the \emph{causal graph} over $V \cup U$.
A directed edge $V_i\to V_j$ implies that $V_i$ directly appears in the structural equation for $V_j$. 
Thus $V_i\rightarrow V_j$ iff $V_i\in Pa_j$. The set $Pa_j$ is called the parent set of $V_j$. We assume this directed graph is acyclic (DAG). 
The causal structure over observed vertex set $\mathcal{V}$ is represented by the latent projection ADMG $G = (V, E)$. 
Under the semi-Markovian assumption, each unobserved confounder can appear in the equation of exactly two observed variables. We represent the existence of an unobserved confounder $[U=U_X= U_Y]\in \mathcal{U}$ between $X,Y$ in the SCM with a bidirected edge $X\leftrightarrow Y$ to the causal graph. These graphs are no longer DAGs although still acyclic. $V_i$ is called an ancestor for $V_j$ if there is a directed path from $V_i$ to $V_j$. Then $V_j$ is said to be a descendant of $V_i$. The set of ancestors of $V_i$ in graph $G$ is shown by $An_G(V_i)$.
We define the partial order induced by $G$ as $\preceq_G$, where
$V_i \preceq_G V_j$ if $V_i$ is an ancestor of $V_j$ in $G$.
\end{definition}

\begin{definition}[Augmented graph]\label{def:aug}
For a set of root causes $R \subseteq \V$, let $G \cup F(R)$ be the 
mDAG
on
nodes $\V \cup \{F\}$ obtained from $G$ by adding a visible root $F$ and a
directed edge $F \to V_i$ for each $V_i \in R$. We use unrestricted
soft-intervention semantics: each $V_i$ is an arbitrary function of its parents
in $G$, of $F$ when $V_i \in R$, and of an independent noise term. 
In particular a targeted mechanism is allowed to ignore $F$ (the no change intervention). 
\end{definition}

\section{Identifying Root Causes under Unobserved Confounding}

\textbf{Problem Statement:}
We consider a reference (normal) environment represented by a structural causal model $\mn = (\V, \mathbf{U}, \mathcal{F}_n, P_n)$ and an anomalous environment represented by another SCM $\ma = (\V, \mathbf{U}, \mathcal{F}_a, P_a)$, in which some mechanisms $f_i^n \in \mathcal{F}_n$ have shifted to corresponding $f_i^a \in \mathcal{F}_a$. We refer to the set of variables whose mechanisms have shifted as the root causes $R^*$. We observe two datasets, $D_n$ and $D_a$, sampled from the normal and anomalous distributions $P_n(\V)$ and $P_a(\V)$, respectively. Following \citet{ikram2022root}, we introduce a binary indicator variable $F$ as an additional node in the causal graph to represent the normal ($F=0$) and anomalous ($F=1$) environments, and add an edge from $F$ to each variable in $R^*$. Our goal is to detect $R^*$.
\subsection{
Root-Cause Solution Sets with a Known Graph
}

We now characterize root causes under unobserved confounding. Without unobserved confounders, or when they appear only in specific graph structures, the minimal root cause set is unique; in many confounded settings, however, multiple root cause sets are consistent with $(P_n, P_a)$.

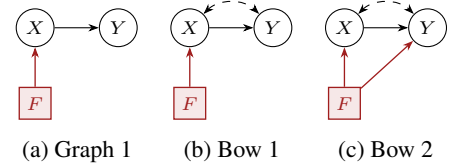
\begin{wrapfigure}{r}{0.44\textwidth}
\vspace{-7mm}
\centering
\begin{minipage}[t]{0.32\linewidth}
\centering
\begin{tikzpicture}
  \node[cnode] (X) at (0,0)    {$X$};
  \node[cnode] (Y) at (1.1,0)  {$Y$};
  \node[fnode] (F) at (0,-1.0) {$F$};
  \draw[dedge] (X) -- (Y);
  \draw[fedge] (F) -- (X);
  \draw[bidir, draw=none] (X) to[bend left=45] (Y);
\end{tikzpicture}\\[1mm]
{\footnotesize (a) Graph 1}
\end{minipage}\hfill
\begin{minipage}[t]{0.32\linewidth}
\centering
\begin{tikzpicture}
  \node[cnode] (X) at (0,0)    {$X$};
  \node[cnode] (Y) at (1.1,0)  {$Y$};
  \node[fnode] (F) at (0,-1.0) {$F$};
  \draw[dedge] (X) -- (Y);
  \draw[bidir] (X) to[bend left=45] (Y);
  \draw[fedge] (F) -- (X);
\end{tikzpicture}\\[1mm]
{\footnotesize (b) Bow 1}
\end{minipage}\hfill
\begin{minipage}[t]{0.32\linewidth}
\centering
\begin{tikzpicture}
  \node[cnode] (X) at (0,0)    {$X$};
  \node[cnode] (Y) at (1.1,0)  {$Y$};
  \node[fnode] (F) at (0,-1.0) {$F$};
  \draw[dedge] (X) -- (Y);
  \draw[bidir] (X) to[bend left=45] (Y);
  \draw[fedge] (F) -- (X);
  \draw[fedge] (F) -- (Y);
\end{tikzpicture}\\[1mm]
{\footnotesize (c) Bow 2}
\end{minipage}
\vspace{-1mm}
\caption{(a) No confounding. (b)–(c) Bow graphs with an unobserved confounder between $X$ and $Y$ (dashed $\leftrightarrow$).}
\label{fig:three_graph}
\vspace{-5mm}
\end{wrapfigure}

\begin{example}[Intuitive example]
\label{ex:dist-constr}
First consider Graph 1, which has no unobserved confounders. The conditional
distribution $P(Y \mid \mathrm{Pa}(Y))$ represents the causal mechanism of
$Y$. Under the independent causal mechanisms assumption (a shift in one
mechanism does not affect another), we can detect a shift in $f_V$ by checking
whether $P_n(V \mid \mathrm{Pa}(V)) = P_a(V \mid \mathrm{Pa}(V))$ across the
two environments. In this case, we recover the true root causes
$\R^* = \{X\}$.
Now consider Bow 1 (Figure~\ref{fig:three_graph}b), where an unobserved
variable $U$ causes both $X$ and $Y$. The ground-truth root cause set is
$\R^* = \{X\}$. However, even though $Y$ is not a root cause, a mechanism
shift in $f_X$ can change $Y$'s conditional distribution in the anomalous
environment, i.e., $P_n(Y \mid X) \neq P_a(Y \mid X)$.
\textbf{CI test:} Given the normal and anomalous datasets,
\textit{conditional independence tests} cannot determine whether $Y$ is a
root cause, because the change in $Y$'s distribution can also be produced by
a pair of normal and anomalous SCMs in which $f_Y$ shifts as well. Hence
$\prc = \{X\}$ and $\prc = \{X, Y\}$ are equally plausible root cause
solutions, both compatible with the given data and graph:
$\mathrm{SolutionSet(CI\text{-}tests)} = \{\{X\}, \{X,Y\}\}$.

\textbf{Instrumental test:} We can further check whether the input
distribution satisfies additional distributional constraints, such as
instrumentality, and use them to shrink the solution set. A variable $F$ is a
valid instrument for the causal effect of $X$ on $Y$ if (i) $F$ affects $X$,
(ii) $F$ affects $Y$ only through $X$, and (iii) $F$ is independent of the
common causes $U$ of $X$ and $Y$.
We can use Pearl's instrumental inequality conditions~\citet{pearl1995testability} to check if a variables is a valid instrument. For binary $F$, $X$, and $Y$: $ P(Y{=}y, X{=}x \mid F{=}0) + P(Y{=}1{-}y, X{=}x \mid F{=}1) \le 1,
   \forall\, x, y \in \{0,1\}.$  
A violation of any of the four inequalities implies
that $F$ is not a valid instrument for the effect of $X$ on $Y$.
Thus by contraposition, an edge
$F \rightarrow Y$ must be present (since i, iii hold by construction,  violation must come from ii).
Therefore $\{X\}$ alone cannot be a
solution, which yields
$\mathrm{SolutionSet(CI\text{-}tests +\ IV\text{-}fail)} = \{\{X,Y\}\}$.
\end{example}
The next natural question is i) how do we characterize our solution set and ii) how can we enforce constraints to reduce the set.
Two rc solutions $R_1, R_2$ represent two different augmented graphs $G \cup R_1$ and  $G \cup R_2$ which can represent distribution set $\mathcal{M}(G \cup R_1)$ and 
$\mathcal{M}(G \cup R_2)$. Informally, if the input normal-anomalous distribution $(P_n, P_a)$ belongs to $\mathcal{M}(G \cup R_1) \cap \mathcal{M}(G \cup R_2)$, we can say that it can be expressed by both augmented graph. Thus, both $R_1, R_2$ are rc solutions for $(P_n, P_a)$. Below, we make the concept of solution set precise.

\subsubsection{The root-cause solution set}

\begin{definition}[Observational equivalence class $\M(G)$]
The set of distributions over the visible variables realizable by
$G$ under classical unrestricted semantics with latents of arbitrary
cardinality.
\end{definition}

\begin{definition}[{Realizable tuple}]\label{def:realizable}
A pair $(\Pn, \Pa)$ of distributions over $\V$ is \emph{realized by} $(G, R)$
if there is a joint $P \in \M\bigl(G \cup F(R)\bigr)$ over $(\V, F)$ with $P(\V \mid F=0) = \Pn(\V)$
and 
$P(\V \mid F=1) = \Pa(\V).$
Since $F$ is a root, its marginal is a free parameter; we only require it to
give both regimes positive probability so the conditionals are defined.
\end{definition}

\begin{definition}[Structural dominance, $G_1 \subseteq G_2$]
\label{def:str-dom}
$G_2$ structurally dominates $G_1$ if every node, every directed edge and every 
latent facet
of $G_1$ is present in $G_2$.
\end{definition}

\begin{definition}[Observational dominance]
    For two mDAGs $G_1, G_2$ on the same nodes,
$G_2$ \emph{observationally dominates} $G_1$ when $\M(G_1) \subseteq \M(G_2)$
at every cardinality of the visible variables 
\end{definition}

\begin{lemma}[Structural dominance $\Rightarrow$ Observational dominance~\citet{ansanelli2025observational}]
\label{lem:struct-dom-obs-dom}
Let $G_1$ and $G_2$ be two mDAGs with the same sets of nodes.
If $G_2$ structurally dominates $G_1$, then it also
observationally dominates it, i.e., 
$\M(G_1) \subseteq \M(G_2)$
\end{lemma}
Figure~\ref{fig:three_graph} shows two graphs: $G_1: G \cup F(\{X\})$ and $G_2: G \cup F(\{X,Y\})$. Since $G_2$ structurally dominates $G_1$ due to the extra $F\rightarrow Y$ edge, $G_2$ realizes a larger set of distributions than $G_1$ : 
$\mathcal{M}(G \cup F(\{X\})) \subseteq \mathcal{M}(G \cup F(\{X,Y\}))$. This implies that if a specific normal and anomalous distribution pair $(P_n, P_a)$ is realized by $G_1$, it will be realized by $G_2$ as well. Thus, if $\{X\}$ is a root cause solution, both $\{X,Y\}$ is also valid root cause solution.
We formally define the solution set as:
\begin{definition}[{Root cause solution set}]\label{def:possrc}
For the observed tuple $(\Pn, \Pa)$ and a graph $G$, define
\begin{equation}
\begin{split}
\PossRC(G, P_n, P_a) &\coloneqq
\bigl\{\, R \subseteq \V : (\Pn, \Pa) \text{ realized by } (G, R) \,\bigr\},
\\
\rc(G,P_n, P_a) &\coloneqq \min_{R \in \PossRC(G,P_n, P_a)} |R| \;\in\; \{0,\dots,n\} \cup \{\infty\},
\end{split}
\end{equation}
with $\rc(G, P_n, P_a) = \infty$ when $\PossRC(G, P_n, P_a) = \emptyset$. If the context is clear, we remove the distribution tuple from the notation.
\end{definition}
For a graph $G$, $\PossRC(G)$ denotes the family of root-cause sets that explain the data, and the \emph{root-cause number} $\rc(G) = \min_{R \in \PossRC(G)} |R|$ is the size of the smallest such set; the inclusion-minimal members of $\PossRC(G)$ are the candidate reported sets. Since the data are generated by the true graph $\Gtrue$ and root-cause set $\Rtrue$, we have $\Rtrue \in \PossRC(\Gtrue)$ and hence $\rc(\Gtrue) \le |\Rtrue|$.

\begin{restatable}%
{lemma}{lemupclosure}
\label{lem:upclosure}
$\PossRC(G)$ is upward closed: if $R\! \in\! \PossRC(G)$ and $R \!\subseteq\! R'\! \subseteq\! \V$, then $R' \!\in\! \PossRC(G)$.
\end{restatable}

\subsubsection{Exploiting Distributional Constraints for Solution-set Reduction}
Having defined the solution set, we determine its members by testing, for each candidate $R$, the distributional constraints implied by the augmented graph $G \cup F(R)$ against the input distributions: as in Example~\ref{ex:dist-constr}, any violated constraint rejects $G \cup F(R)$ and removes $R$ from $\PossRC$.
In general, the observational distributions compatible with a latent-variable causal model are restricted by two families of constraints~\citep{evans2023latentfree, ansanelli2025observational}: equality constraints, such as conditional independences~\citep{pearl2009causality} and nested Markov (Verma) constraints~\citep{verma1990equivalence, richardson2023nested}, and inequality constraints, such as the instrumental inequality~\citep{pearl1995testability}, e-separation inequalities~\citep{evans2012graphical}, and Bell inequalities~\citep{bell1964epr}.
Enumerating and testing these constraints individually is non-trivial, particularly since inequality constraints are difficult to derive in general. We bypass this by directly searching for an SCM that is consistent with $G \cup F(R)$ and reproduces both the normal and anomalous distributions; the existence of such an SCM implies that all constraints are satisfied.
We establish a connection between $\PossRC$ and normal anomalous SCMs as follows:
\begin{definition}[SCM Pairset, $\widehat{M}_{na}(G^*, \PossCn)$]
For any $\PossCn \subseteq \mathbf{V}$, there exists a normal and anomalous SCM pair $\widehat{\mathcal{M}}_n = \big(\{\widehat f^{\,n}_X\}_{X \in \mathbf{V}},\,
\widehat P(U)\big)$ and
$\widehat{\mathcal{M}}_a = \big(\{\widehat f^{\,a}_X\}_{X \in \mathbf{V}},\,
\widehat P(U)\big)$ over $\mathbf{V}$ sharing the exogenous distribution
$\widehat P(U)$ such that $(i)$ $\widehat{\mathcal{M}}_n$ and $\widehat{\mathcal{M}}_a$ have the true
    ADMG graph $G^*$; and $(ii)$ For all variables $V\! \in \!\PossCn$, mechanisms change arbitrarily (including no shift) across environments,
    $\widehat f^{\,n}_V(\mathrm{pa}(V), U_V)
\!     \rightarrow \!\widehat f^{\,a}_V(\mathrm{pa}(V), U_V)$
    while invariant mechanisms  
     $
    \widehat f^{\,n}_V(\mathrm{pa}(V), U_V)
   \!=\! \widehat f^{\,a}_V(\mathrm{pa}(V), U_V) ,
    $
    for all $V \in \mathbf{V} \setminus \PossCn$.
    We define the set of all such %
    SCM pairs as SCM Pairset $\widehat{M}_{na}(G^*, \PossCn)$.
\end{definition}

\begin{definition}[Population root-cause loss, $\ell^*$]
\label{def:score}
For $\PossCn \subseteq \mathbf{V}$, the \emph{population root-cause loss} is the least
distributional mismatch attainable by a candidate pair,
\begin{equation}
\label{eq:inf-loss}
\ell^*(G^*, \PossCn) \;=\;
\inf_{(\widehat{\mathcal{M}}_n, \widehat{\mathcal{M}}_a)\in  \widehat{M}_{na}(G^*, \PossCn)}
\Big[\, d\big(\widehat P_n,\, P^*_n\big)
                  + d\big(\widehat P_a,\, P^*_a\big) \,\Big]
\end{equation}
where $\widehat P_n$ and $\widehat P_a$ are the distributions over $\mathbf{V}$
induced by $\widehat{\mathcal{M}}_n$ and $\widehat{\mathcal{M}}_a$.
$\widehat{M}_{na}(G^*, \PossCn)$ is the SCM pairs, and $d$ is a discrepancy on distributions over
$\mathbf{V}$ with $d(P,Q) = 0 \iff P = Q$.
\end{definition}

\subsection{RCA-DCM: Root-Cause Detection via SCM Search}

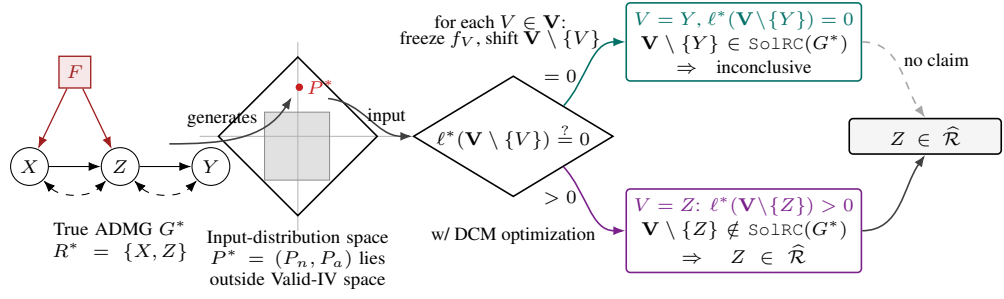
\begin{figure*}[t]
\centering
\definecolor{rcnodeblue}{RGB}{47,85,151}
\definecolor{rcnodeorange}{RGB}{197,90,17}
\definecolor{rcnodegray}{RGB}{115,115,115}
\definecolor{rcvalid}{RGB}{225,225,225}
\definecolor{rchit}{RGB}{200,30,30}
\definecolor{rcstepone}{RGB}{122,32,138}
\definecolor{rcsteptwo}{RGB}{0,112,102}
\begin{tikzpicture}[
  x=1cm, y=1cm,
  >={Latex[length=1.6mm]},
  font=\scriptsize,
  rcobs/.style  ={circle, draw=black, thick, fill=rcnodeblue,   text=white, minimum size=5.4mm, inner sep=0pt, font=\scriptsize},
  rcreg/.style  ={circle, draw=black, thick, fill=rcnodeorange, text=white, minimum size=5.4mm, inner sep=0pt, font=\scriptsize},
  rcedge/.style ={->, semithick, draw=rcnodeblue!80!black},
  rcfedge/.style={->, semithick, draw=rcnodeorange!85!black},
  rcbidir/.style={<->, semithick, dashed, draw=rcnodegray},
  rcdec/.style  ={diamond, draw=black, semithick, fill=white, aspect=1.7,
                  minimum width=2.5cm, minimum height=1.6cm, inner sep=0pt, align=center},
  rcbox/.style  ={rounded corners=1.5pt, draw, semithick, fill=white,
                  align=center, text width=2.9cm, inner sep=3pt},
  rcflow/.style ={->, semithick, draw=black!75},
  rcnote/.style ={align=center, text=black, anchor=north}
]

\node[fnode] (F) at (1.05, 1.15) {$F$};
\node[cnode] (X) at (0.45,-0.10) {$X$};
\node[cnode] (Z) at (1.65,-0.10) {$Z$};
\node[cnode] (Y) at (2.85,-0.10) {$Y$};
\draw[fedge] (F) -- (X);
\draw[fedge] (F) -- (Z);
\draw[dedge] (X) -- (Z);
\draw[dedge] (Z) -- (Y);
\draw[bidir] (X) to[bend right=45] (Z);
\draw[bidir] (Z) to[bend right=45] (Y);
\node[rcnote, text width=3.0cm] at (1.65,-0.70)
  {True ADMG $\Gtrue$\\ $\Rtrue=\{X,Z\}$};

\fill[rcvalid] (3.56,-0.28) rectangle (4.42,0.62);
\draw[black!30, line width=0.3pt] (2.70,0.30) -- (5.30,0.30);
\draw[black!30, line width=0.3pt] (4.00,-0.95) -- (4.00,1.55);
\draw[semithick] (2.95,0.30) -- (4.00,1.35) -- (5.05,0.30) -- (4.00,-0.75) -- cycle;
\draw[black!45, line width=0.3pt] (3.56,-0.28) rectangle (4.42,0.62);
\fill[rchit] (4.02,0.95) circle (1.4pt);
\node[rchit, anchor=west, inner sep=1pt] at (4.08,0.95) {$P^*$};
\node[rcnote, text width=3.0cm] at (4.00,-0.82)
  {Input-distribution space\\ $P^*=(P_n,P_a)$ lies outside Valid-IV space};

\node[rcdec] (test) at (6.85, 0.30) {$\ell^*( \mathbf{V}\setminus \{\Rv\})\overset{?}{=}0$};
\node[align=center, anchor=south] at (6.68, 1.34)
  {for each $V\in\mathbf{V}$:\\[-2pt] freeze $f_V$, shift $\mathbf{V}\setminus\{V\}$};
\node[align=center, anchor=north, text=black, font=\scriptsize]
  at (6.85, -0.78) {w/ DCM optimization};
\node[rcbox, draw=rcsteptwo] (top) at (9.90, 1.55)
  {\textcolor{rcsteptwo}{$V=Y$, $\ell^*(\mathbf{V}\setminus\{Y\})=0$}\\[1pt]
   $\mathbf{V}\setminus\{Y\}\in\PossRC(\Gtrue)$\\[1pt]
   $\Rightarrow\;$ inconclusive};
\node[rcbox, draw=rcstepone] (bot) at (9.90,-0.95)
  {\textcolor{rcstepone}{$V=Z$: $\ell^*(\mathbf{V}\setminus \{Z\})>0$}\\[1pt]
   $\mathbf{V}\setminus\{Z\}\notin\PossRC(\Gtrue)$\\[1pt]
   $\Rightarrow\; Z\in\prc$};

\node[rcbox, text width=1.8cm, draw=black, fill=black!4] (out) at (12.30, 0.30)
  {$Z\in\prc$\\[1pt] };

\draw[rcflow] (2.30,0.2) to[out=0,in=238] (3.92,0.82);
\node[anchor=north, inner sep=1pt] at (3.00,0.66) {generates};
\draw[rcflow] (4.40,0.80) to[out=-6,in=180]
  node[pos=0.72, above, inner sep=1.5pt] {input} (test.west);
\draw[rcflow, draw=rcsteptwo] (test.north east) to[out=48,in=180]
  node[pos=0.26, above left, inner sep=0.5pt] {$=0$} (top.west);
\draw[rcflow, draw=rcstepone] (test.south east) to[out=-48,in=180]
  node[pos=0.26, below left, inner sep=0.5pt] {$>0$} (bot.west);
\draw[rcflow, draw=black!35, dashed] (top.east) to[out=0,in=125]
  node[pos=0.52, above right, inner sep=1pt] {no claim} (out.north);
\draw[rcflow] (bot.east) to[out=0,in=235]
  node[pos=0.52, below right, inner sep=1pt] {} (out.south);
\end{tikzpicture}
\caption{%
\textbf{Workflow example.} In the true augmented graph the fault shifts $X$
and $Z$, so $\Rtrue=\{X,Z\}$, and latent confounders act on both
$X\leftrightarrow Z$ and $Z\leftrightarrow Y$. The induced pair
$P^*=(P^*_n,P^*_a)$ falls outside the shaded Valid-IV region of the input
space: it violates the instrumental inequalities for the effect of $X$ on $Z$,
which rules out $F$ being excluded from $Z$ (exclusion condition) and hence forces the edge
$F\rightarrow Z$. Each candidate $V$ is tested by freezing $f_V$ while the
mechanisms on $\mathbf{V}\setminus\{V\}$ shift freely. The violation makes
$\mathbf{V}\setminus\{Z\}$ infeasible, so
$\ell^*(\mathbf{V}\setminus\{Z\})>0$ certifies that $Z$ lies in every
$\PossCn\in\PossRC(\Gtrue)$ and hence in $\prc$. For $V=Y$ the score attains
zero, which is inconclusive: it is consistent with
$\mathbf{V}\setminus\{Y\}\in\PossRC(\Gtrue)$, and we draw no conclusion
from it.}
\label{fig:invalid-iv-instance}
\end{figure*}

\textbf{Brute Force algorithm:}
To obtain the members of $\PossRC$, we can: 1. Iterate over all possible $\PossCn \subseteq \mathbf{V}$ as a candidate solution 2. For each $\PossCn$, search for SCM Pairset $\widehat{M}_{na}(G^*, \PossCn)$ such that $\ell^*(G^*, \PossCn)=0$ 3. If there exists at least one SCM pair $(\widehat{\mathcal{M}}_n, \widehat{\mathcal{M}}_a) \in \widehat{M}_{na}$ found, we add $\PossCn$ to $\PossRC$ as a candidate solution.
Nevertheless, we have two challenges: i) this approach requires iteration over $2^{|\mathbf{V}|}$ number of sets and ii) how to find such an SCM pair - is non-trivial.
To efficiently find such pair with stochastic gradient descent, we learn a proxy of the SCMs with deep causal models (Definition~\ref{def:SCM}) and optimize following our constraints.
\begin{definition}[Deep causal generative models (DCM)~\citep{kocaoglu2018causalgan,xia2021causal,rahman2024modular}]
\label{def:SCM}
	A neural net architecture $\mathbb{G}$ is called a deep causal generative model (DCM) for an ADMG $G=(\mathcal{V},\mathcal{E})$ if it is composed of a collection of neural nets, one  $f_i$ (or interchangeably $f_{V_i}$) for each $V_i\in\mathcal{V}$ such that 
		i) \emph{each $f_i$ accepts a sufficiently high-dimensional noise vector $N_i$,} 
		ii) \emph{the output of $f_j$ is input to $f_i$ iff $V_j\in Pa_G(V_i)$,}
		iii) \emph{$N_i= N_j$ iff $V_i\leftrightarrow V_j$. }
        {A DCM is trained to learn a proxy of the true SCM.}
DCM generators are represented as $\mathbb{G}=\{f_{1},...,f_{n}\}$ parameterized by $\Theta= \{\theta_1, ... , \theta_{|\mathbf{V}|}\}$ where $n=|\mathbf{V}|$.
Similar to the original data distribution, $P(\mathbf{V})$, we define $\widehat{P}(\mathbf{V})$ to be the distribution induced by the $\theta$ parameterized DCM.
Noise vectors $N_i$ replace both the exogenous noises and the unobserved confounders in the true SCM. They are of sufficiently high dimension to induce the observed distribution. We say that a DCM is \emph{representative enough for an ADMG} if the neural networks have sufficiently many parameters to induce any observed distribution induced by any SCM that entails the ADMG. 
Let $v=[v_1, v_2,..., v_{n}]$ st $V_i \in \mathbf{V}$.  $v\sim P(\mathbf{V})$ is real samples and $\hat{v}\sim \hat{P}(\mathbf{V})$ is DCM generated samples:
$\hat{v}_i = f_i(\hat{pa}(V_i), {N_i});  f_i \in \{f_V: V\in \mathbf{V}\}$. A discriminator compares  $v$ and $\hat{v}$ to train $\forall f_i$.
\end{definition}

Our aim is to now construct the $\PossRC$ from 
a pool of exponentially many candidate subsets by searching for SCMs that are consistent with the causal graph $G$ and satisfy the distributional constraints. We circumvent this problem by iterating over variables in $\mathbf{V}$ and searching for SCMs
assuming $\Yv$ is not a root cause, i.e., the absence of the edge $F \rightarrow R$. We test whether there exist two SCMs consistent with the causal graph $G^*$, allowing changes in all mechanisms except $\Yv$, can shift from normal $P_n$ to anomalous $P_a$ distribution. If no such SCM exists, the input distribution must violate at least one equality or inequality constraint implied by the absence of $F \rightarrow R$.  We therefore conclude that $R$ must be a root cause, without iterating many candidates and identifying the violated constraint explicitly. We claim the following lemma.

\begin{restatable}{lemma}{lemltosolrc}
\label{lemm-ltoSolrc}
Suppose $\mathcal D_n, \mathcal D_a$ are generated by $(\mathcal M_n^*, \mathcal M_a^*)$ with common ADMG $G^*$, and let latent invariance hold.  For any $\Yv \in \mathbf{V}$, 
if $\ell^*(G^*, \mathbf{V}\setminus \{\Yv\} ) > 0 $ then $\Rightarrow   \mathbf{Z} \setminus \{ \Yv \} \notin \PossRC(G^*, P_n^*, P_a^*) \text{ for any } \mathbf{Z}\subseteq \mathbf V .$ 
\end{restatable}

\begin{restatable}{theorem}{thmrcsound}
\label{th-main:rc_sound}
Under Assumption~\ref{asm:dcm} (Expressive DCM),
For any $\Yv \in \mathbf{V}$, 
if $\ell(G^*,  \mathbf{V}\setminus \{\Yv\} ) > 0 $ then $\Yv \in R^*$.
\end{restatable}

Intuitively, if there exists no SCM pairs that achieves $\ell^*(G^*,Y)=0$, the no solution set without $Y$ exists. $Y$ must belong to all solution set $\PossCn \in \PossRC$ and thus a true root cause. This reduces $2^{|\mathbf{V}|}$ bruteforce  iterations to $|\mathbf{V}|$ iterations.

\begin{restatable}[{Soundness}]{proposition}{propsoundfree}
\label{prop:sound-free}
Under Assumption~\ref{asm:dcm},
  $\widehat{R} \;=\; \bigcap \PossRC(\Gtrue) \;\subseteq\; \Rtrue$
\end{restatable}
\begin{wrapfigure}{r}{0.50\linewidth}
\vspace{-8mm}
\begin{minipage}{1\linewidth}
\begin{algorithm}[H]
\small
\caption{(Input: $G$, $D_n \!\sim \!P_n(\mathbf{V})$, $D_a\! \sim \!P_a(\mathbf{V})$)}
\label{alg:dcm-rca}
\begin{algorithmic}[1]
\FOR{each $Y \in \mathbf{V}$ in parallel}
    \STATE Initialize $\mn$ and $\ma$.
    \label{init}
    \FOR{Run for N epochs}
    \STATE Use $\mn$ to generate samples from $\widehat{P}_n(\mathbf{V})$.
    \label{train1}
    \STATE Copy $f_Y$ from $\mn$ to $\ma$ and freeze it.
    \label{copy}
    \STATE Use $\ma$ to generate samples from $\widehat{P}_a(\mathbf{V})$.
    \label{train2}
    \STATE Backprop:
    $L =
    L_1(\hat{P}_n, P_n)
    +
    L_2(\hat{P}_a, P_a)$
    \STATE $score[Y] = L$
    \label{score}
    \ENDFOR
\ENDFOR
\STATE Sort $\textit{score}$ and return as $\textit{Rank}$.
\label{rank}
\end{algorithmic}
\end{algorithm}
\end{minipage}
\vspace{-5mm}
\end{wrapfigure}
Now, the optimization in Equation~\ref{eq:inf-loss}: SCM search is over a set with infinitely many members where each member is a SCM pair consistent with the causal graph $G^*$, allowing changes in all mechanisms except $Y$, can shift from normal $P_n$ to anomalous $P_a$ distribution. 
We provide our algorithm workflow for an example graph in Figure~\ref{fig:invalid-iv-instance} and our pseudo-code in Algorithm~\ref{alg:dcm-rca}.
For each variable, we can run the algorithm independently in parallel.
For any variable $Y\! \in\! \mathbf{V}$, we first initialize two deep causal models $(\widehat{\mathcal{M}}_n, \widehat{\mathcal{M}}_a)$ to represent the normal and anomaly SCM (Line~\ref{init}).
We use causal normalizing flows~\citep{javaloy2023causal} as the DCM
backbone.
In Line~\ref{train1},  $\widehat{\mathcal{M}}_n$ generates samples from the normal joint distribution $\hat{P}_n(\mathbf{V})$. Next, we copy 
the model $f_Y(.)$ 
from $\hat{\mathcal{M}}_n$ to $\hat{\mathcal{M}}_a$ (Line~\ref{copy}).
In Line~\ref{train2}, we employ $\hat{\mathcal{M}}_a$ to 
generate samples from anomaly joint distribution $\hat{P}_a(\mathbf{V})$.
We calculate $L$ by comparing 
the generated samples with normal and anomaly dataset
and backpropogate on $L$.
The losses $L_1$ and $L_2$ are estimated empirically using real samples from $D_n,D_a$ and generated samples from $\widehat{P}_n,\widehat{P}_a$.
We store $L$ in $\textit{score}$ and rank them in Line~\ref{rank} to obtain the nodes with highest mechanism shift at the top.
{We learn proxy to the structural function in the design deep causal models and arrange them according to the causal graph}. We execute Algorithm~\ref{alg:run-nf} to generate samples from $\mathcal{M}$. Finally, we use 
distance metric to compare the dissimilarity between learned and true distribution.

\textbf{Error analysis of the \myalgo{} ranking:}
Theorem~\ref{th-main:rc_sound} states that if $\ell(G^*, \mathbf{V}\setminus \{Y\}) > 0$, then $Y$ is a root cause. In practice, however, even with perfect model training, the finite sample size leads us to observe $\ell(\cdot) > 0$ for every variable.
In Appendix~\ref{appex:erro-anal}, we show that under a mild condition, every root cause is scored strictly above every non-root cause, so the $k = |R^*|$ root causes occupy the top $k$ positions of the ranking returned by \myalgo{}. 

\subsection{Theoretical Guarantees under Graph Mis-specification}
\label{sec:graph-mis}

In the previous section, we provided \myalgo to find root causes in a system containing unobserved confounders given the true causal graph (ADMG) as input. However, it assumes access to the true ADMG. In most real-world scenarios, it is difficult to find such a graph without domain knowledge. Thus, in this section, we analyze what \myalgo outputs when this assumption is violated and we have a sparser or denser graph compared to the true ADMG.
The proofs are provided in Appendix~\ref{sec:proof-mis}.

First, we answer an important question: can we obtain the true root cause set from an rca algorithm for any arbitrary mis-specified graph? We formally prove that it is only possible when the normal and anomalous distributions generated from the SCM with true grahp $G^*$ are realized by the mis-specified graph $G'$ as well.
Note that exploring arbitrary different graph $G'$ is outside the scope of this paper and we enlist this as a limitation of the current work.
Lemma~\ref{lem-main:finite} formalizes this.

\begin{lemma}\label{lem-main:finite}
Let $G^*,G'$ be the true and mis-specified graphs,  either by adding or by removing edges. Given $(\Pn, \Pa)$ realized by  $(G^*, R^*)$, we have
$\PossRC(G') \neq \emptyset $ iff  $\Pn, \Pa \in \M(G')$. 
\end{lemma}

\begin{corollary}
    $\rc(G) < \infty$ if and only if $\Pn \in \M(G)$ and $\Pa \in \M(G)$.
\end{corollary}

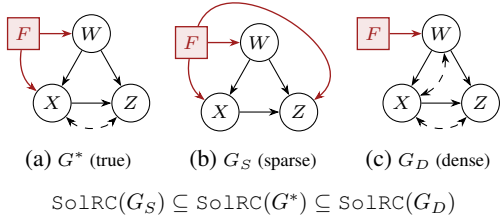
\begin{figure}[t]
\centering
\begin{minipage}[c]{0.48\textwidth}
\centering
\begin{subfigure}[b]{0.31\linewidth}
\centering
\begin{tikzpicture}
  \node[fnode] (F) at (-0.90, 0.00) {$F$};
  \node[cnode] (W) at ( 0.00, 0.00) {$W$};
  \node[cnode] (X) at (-0.52,-0.92) {$X$};
  \node[cnode] (Z) at ( 0.52,-0.92) {$Z$};
  \draw[fedge] (F) -- (W);
  \draw[fedge] (F) to[bend right=30] (X);
  \draw[dedge] (W) -- (X);
  \draw[dedge] (W) -- (Z);
  \draw[dedge] (X) -- (Z);
  \draw[bidir] (X) to[bend right=45] (Z);
\end{tikzpicture}
\caption{\scriptsize $G^*$ (true)}
\label{fig:mis-true-graph}
\end{subfigure}
\hfill
\begin{subfigure}[b]{0.31\linewidth}
\centering
\begin{tikzpicture}
  \node[fnode] (F) at (-0.90, 0.00) {$F$};
  \node[cnode] (W) at ( 0.00, 0.00) {$W$};
  \node[cnode] (X) at (-0.52,-0.92) {$X$};
  \node[cnode] (Z) at ( 0.52,-0.92) {$Z$};
  \draw[fedge] (F) -- (W);
  \draw[fedge] (F) to[bend right=30] (X);
  \draw[fedge] (F) to[out=68,in=38,looseness=1.35] (Z);
  \draw[dedge] (W) -- (X);
  \draw[dedge] (W) -- (Z);
  \draw[dedge] (X) -- (Z);
\end{tikzpicture}
\caption{\scriptsize $G_S$ (sparse)}
\label{fig:mis-sparse-graph}
\end{subfigure}
\hfill
\begin{subfigure}[b]{0.31\linewidth}
\centering
\begin{tikzpicture}
  \node[fnode] (F) at (-0.90, 0.00) {$F$};
  \node[cnode] (W) at ( 0.00, 0.00) {$W$};
  \node[cnode] (X) at (-0.52,-0.92) {$X$};
  \node[cnode] (Z) at ( 0.52,-0.92) {$Z$};
  \draw[fedge] (F) -- (W);
  \draw[dedge] (W) -- (X);
  \draw[dedge] (W) -- (Z);
  \draw[dedge] (X) -- (Z);
  \draw[bidir] (X) to[bend right=45] (Z);
  \draw[bidir] (W) to[bend left=38] (X);
\end{tikzpicture}
\caption{\scriptsize $G_D$ (dense)}
\label{fig:mis-dense-graph}
\end{subfigure}
\\[3pt]
\resizebox{0.8\linewidth}{!}{$\PossRC(G_S)\subseteq\PossRC(\Gtrue)\subseteq\PossRC(G_D)$}
\end{minipage}\hfill
\begin{minipage}[c]{0.48\textwidth}
\caption{\footnotesize
Augmented graphs of two misspecifications of the true graph, with
$G_S\subset G^*\subset G_D$. All $F$-edges are in \textcolor{Fclr}{red}.
While $R(G^*)=\{W,X\}$ is ground truth, we obtain the superset $\hat R(G_S)=\{W,X,Z\}$ for $G_S$, and the subset $\hat R(G_D)=\{W\}$ for $G_D$.}
\label{fig:graphs}
\end{minipage}
\end{figure}

Consider the sparse and dense misspecifications $G_S$ and $G_D$ of the 3-node true graph $G^*$ in Figure~\ref{fig:graphs}, and suppose both can realize the input pair $(P_n, P_a)$, so that each yields a non-empty $\PossRC$. Since the true root cause set is $R^*=\{W,X\}$, we observe shifts from $P_n$ to $P_a$ in $P(W)$, $P(X \!\mid \!W)$, and $P(Z \!\mid \!X,W)$, and $\PossRC(G^*)\!=\!\{\{W,X\},\{W,X,Z\}\}$.
Given the sparse graph $G_S$ (Figure~\ref{fig:mis-sparse-graph}), reproducing all three shifts requires shifting every mechanism $f_W, f_X, f_Z$, i.e., $F\!\rightarrow\!\{W,X,Z\}$, so $\PossRC(G_S)=\{\{W,X,Z\}\}$. Any proper subset, such as $\{W,X\}$, would imply $Z \indep F \mid \{W,X\}$ in $G_S$ and hence $P_n(Z \!\mid \!W,X)=P_a(Z \!\mid \!W,X)$, contradicting the observed shift.
Given the dense graph $G_D$ (Figure~\ref{fig:mis-dense-graph}), Lemma~\ref{lem:main-graphmono} gives $\PossRC(G^*)\subseteq\PossRC(G_D)$. Moreover, for some distributions, shifting $f_W$ alone reproduces all three shifts, so $\PossRC(G_D)\!=\!\{\{W\},\{W,X\},\{W,X,Z\}\}$.
Hence, $\PossRC(G_S)\!\subset\!\PossRC(G^*)\!\subset\!\PossRC(G_D)$: a denser graph can admit smaller root cause sets that are infeasible under the true graph, whereas a sparser graph can exclude the true root cause set and require a larger one.

Given a sparse graph $G_1$, we can add directed (causal relations) or bi-directed edges (latent confounders) to obtain a denser graph $G_2$ such that $G_1$ is
structurally dominated by $G_2$, written $G_1 \!\subseteq\! G_2$. By
Lemma~\ref{lem:struct-dom-obs-dom}, structural dominance implies that $G_2$
can realize any pair $(P_n, P_a)$ generated by $G_1$. Hence, rather than comparing the root causes obtained under a misspecified graph against those obtained under the true graph, we compare how they change between a sparser vs. a denser graph.

\begin{restatable}{lemma}{graphmono}\label{lem:main-graphmono}
If $G_1 \subseteq G_2$, then
$\PossRC(G_1) \subseteq \PossRC(G_2)$.
\end{restatable}
\begin{restatable}{lemma}{master}\label{lem:master}
If $G_1 \subseteq G_2$, then $\rc(G_2, P_n, P_a) \le \rc(G_1, P_n, P_a)$.
\end{restatable}

\section{Experimental Results}
We evaluate \myalgo{} on synthetic, semi-synthetic, and real-world datasets. We compare it against five representative non-causal and causal RCA approaches based on statistical hypothesis testing (BARO~\citep{pham2024baro}), $z$-score-based statistical analysis (NSigma~\citep{li2022causal}), regression-based hypothesis testing (CIRCA~\citep{li2022circa}), causal discovery (RCD~\citep{ikram2022rcd}), and intervention-aware causal inference (RCG~\citep{ikram2025rcg}). We adopt their implementation details and hyperparameters from the RCAEval repository~\citep{pham2025rcaeval}. Code: \url{https://github.com/Musfiqshohan/RCA-DCM}.
\textbf{Metrics:} To evaluate the output ranking, we use 
$T@k \;=\; \frac{1}{|\mathbf{E}|}\sum_{e \in \mathbf{E}}
\frac{\left|\hat{R}_e[1,..,k] \cap R^{*}_{e}\right|}
{\min\!\left(k,\, |R^{*}_{e}|\right)}$, i.e., how many of the true root causes lie in the top $k$, and 
$\mathrm{PRR} \;=\; \frac{1}{|\mathbf{E}|}\sum_{e \in \mathbf{E}}
\mathbf{1}\!\left[\,T@|R^{*}_{e}| = 1\,\right]$, i.e., in how many cases all true root causes are at the top (i.e., the perfect recovery rate).
We provide more experimental details, and discuss the availability of the causal graph (Appendix~\ref{app:train}), sample-size requirements (Appendix~\ref{sec:sample-sensitivity}), and runtime (Appendix~\ref{app:complexity}).

\subsection{Simulated Datasets}
\label{exp:sim-dat}

We follow \citet{yang2024learning} to generate synthetic normal, anomalous data from randomly generated graphs, varying the total
variable count, latent proportion, and edge density; full details %
in Appendix~\ref{app:sim-dat}.

\textbf{Exp 1: Nonlinear model with unobserved confounders:}
We use graphs with $p = 6$ observed and $m = 4$ latent variables (total $n=10$) and vary the confounding strength $\lambda \in \{3, 10\}$, which uniformly scales all latent-to-observed coefficients. We take $\lfloor p/2 \rfloor = 3$ observed variables as ground-truth root causes. 
\textbf{Observation:}
Figure~\ref{fig:latent_str}(a) reports PRR for $\lambda \in \{3, 10\}$. \myalgo{} attains $99\%$ and $84\%$, ahead of every baseline in both settings (per-baseline values are listed in Appendix~\ref{app:sim-dat}). All methods degrade as confounding strengthens: \myalgo{}, starting at $99\%$, and RCD, starting at $74\%$, both drop by $15$ percentage points, while BARO, NSigma, CIRCA, and RCG drop by $26$, $24$, $24$, and $46$ points, respectively. This suggests that \myalgo{} is more robust to strong confounding and better preserves root-cause recovery than the baselines. 

\begin{figure}[t]
     \vspace{-9mm}
    \hspace{-10mm}
        \centering
       \includegraphics[width=0.8\linewidth]{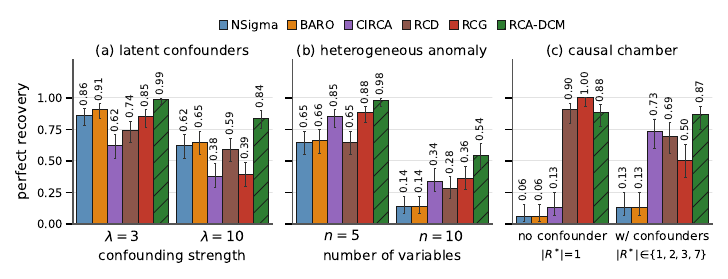}
    \caption{Perfect recovery rate (PRR) with $95\%$ confidence intervals.
(a) increasing latent confounding, (b) heterogeneous anomalies where a non-cause can
out-shift a true cause, (c) real causal-chamber data with the root cause observed
vs.\ hidden. \myalgo{} remains %
strong across all three settings; comparisons
use a paired test, so overlapping intervals do not imply the absence of a difference.
    }
    \label{fig:latent_str}
    \vspace{-3mm}
\end{figure}

\textbf{Exp 2: Nonlinear model with heterogeneous anomalies:}
As discussed in Section~1, baselines that rank variables by marginal shift fail when a non-root-cause descendant out-shifts a root cause (heterogeneous anomalies).
We construct such heterogeneous anomalies by scaling the injected shift magnitude linearly with topological depth, so that shifts accumulate along directed paths and a downstream descendant out-shifts a genuine upstream root cause. In the generated data, this trap occurs (i.e., some non-root cause exhibits a larger marginal shift than some true root cause) in roughly $60\%$ of trials at $p=10$. We vary the number of observed variables $n \in \{5, 10\}$ ($n=p$ as $m=0$). 
\textbf{Observation:}
Figure~\ref{fig:latent_str}(b) shows PRR as the system grows. \myalgo{} achieves the highest rate at both sizes, $98\%$ at $n=5$ and $54\%$ at $n=10$, outperforming all five baselines in both settings. The marginal-shift methods (BARO and NSigma) degrade the most, falling to $14\%$ at $n=10$; this is the failure mode the construction targets, since it makes ranking by shift magnitude wrong in most trials. The graph-based methods (RCG $36\%$, CIRCA $34\%$, RCD $28\%$) also degrade as the system grows. These results suggest that \myalgo{} is more robust to increasing system complexity, where the effects of multiple root causes can compound.

\textbf{Exp 3: Robustness to graph misspecification.}
Perturbing \emph{only} the graph supplied to \myalgo{} on the $100$ datasets with $\lambda = 10$ from \textbf{Exp 1} (true-graph control: PRR of $88\%$ and $84\%$ in two independent runs), adding $50\%$ spurious directed and bidirected edges (PRR $= 84\%$, $p = 0.39$) or deleting $50\%$ of the bidirected edges (PRR $= 87\%$, $p = 1.00$) does not change accuracy significantly under a paired McNemar test; deleting $50\%$ of both directed and bidirected edges degrades it ($76\%$, $p = 0.019$), which is still above every baseline at $\lambda = 10$ (best: BARO, $65\%$).
These results are consistent with Section~\ref{sec:graph-mis} (see Appendix~\ref{app:sim-dat}).

\subsection{Causal Chamber (Real-World Physical Testbed)}
\label{exp:chamber}

\textbf{Exp 4:} We evaluate \myalgo{} on the Causal Chambers benchmark~\citep{gamella2025causal}, a real-world physical testbed. We use the light-tunnel chamber in its standard configuration, which comes with a ground-truth DAG over $38$ variables and $57$ edges. This yields $52$ cases, which we evaluate in two conditions: with the true root cause \emph{observed}, and with it \emph{hidden}, so that it acts as a latent confounder among its former children. In the observed condition, $|R^*| = 1$; in the hidden condition, the hidden cause has between $1$ and $7$ children, all of which must be recovered as root causes. Details are in Appendix~\ref{app:chamber}.
\textbf{Observation:}
Figure~\ref{fig:latent_str}(c) reports PRR in both conditions. When the root cause is observed, there is no confounding and every case has a single root cause; here, \myalgo{} is competitive but does not lead ($88\%$, against $100\%$ for RCG and $90\%$ for RCD). Under this harder setting, the ordering reverses: \myalgo{} attains $87\%$, against $73\%$ for the best baseline, and it is the only method whose performance barely changes between conditions ($-2$ points, against $-50$ for RCG and $-21$ for RCD). The gap is largest on the cases with multiple root causes: on these cases alone, \myalgo{} attains a PRR of $63\%$, while CIRCA reaches $25\%$ and RCG $13\%$.

\subsection{Microservice Datasets}
\label{exp:micro}
\begin{table}[t]
\centering
\setlength{\tabcolsep}{4pt}
\caption{Average Top-$k$ accuracy over the five fault types (best in bold).}
\label{tab:avg-topk}
\footnotesize
\begin{tabular}{@{}l ccc ccc@{}}
\toprule
& \multicolumn{3}{c}{Sock Shop} & \multicolumn{3}{c}{Online Boutique} \\
\cmidrule(lr){2-4}\cmidrule(l){5-7}
Method & T1 & T3 & T5 & T1 & T3 & T5 \\
\midrule
\textbf{DCM} & \textbf{0.88} & \textbf{0.99} & \textbf{0.99} & \textbf{0.78} & 0.90 & \textbf{0.98} \\
NSigma & 0.75 & 0.97 & \textbf{0.99} & 0.70 & \textbf{0.91} & 0.96 \\
BARO   & 0.74 & 0.98 & \textbf{0.99} & 0.62 & 0.90 & 0.94 \\
CIRCA  & 0.65 & 0.96 & \textbf{0.99} & 0.52 & 0.89 & \textbf{0.98} \\
RCD    & 0.38 & 0.54 & 0.58 & 0.49 & 0.59 & 0.62 \\
RCG    & 0.54 & 0.67 & 0.70 & 0.71 & 0.83 & 0.85 \\
\bottomrule
\end{tabular}
\end{table}

\textbf{Exp 5:}
We evaluate \myalgo{} on SockShop, a cloud computing dataset recorded from a microservice-based replica of a web application specifically designed for evaluating root cause analysis methods~\citep{pham2024root}. It contains $125$ anomaly datasets covering five fault types injected into five services. We assume that \emph{no} causal structure is available to \myalgo{}: instead of a call graph, we supply a \emph{confounded sink}, in which every recorded service points to the service under test and a single shared latent confounds all of them (the justification for this choice is given in Appendix~\ref{app:micro}). This places \myalgo{} at a disadvantage relative to two of the baselines: CIRCA and RCG are given the true edge-inverted call graph, while NSigma, BARO, and RCD do not use a graph at all.

\textbf{Exp 6:}
We repeat the evaluation on Online Boutique~\citep{pham2024root}, a second microservice benchmark with the same five fault types, each injected five times, again yielding $125$ anomaly datasets, but over $11$ services and with substantially longer traces (roughly $2100$ normal and $2100$ anomalous observations per case). The graph treatment is identical to Exp 5: \myalgo{} receives only the confounded sink, while CIRCA and RCG receive the true call graph.

\textbf{Results:}
Table~\ref{tab:avg-topk} reports the average top-$k$ accuracy (per-fault results are in Table~\ref{tab:per-fault-topk}). On SockShop, \myalgo{} attains the highest average top-$1$ accuracy, $0.88$, ahead of NSigma ($0.75$), BARO ($0.74$), CIRCA ($0.65$), RCG ($0.54$), and RCD ($0.38$); every gap is significant under a paired test ($p<0.001$). It is best or tied for best on all five fault types, and its errors are near-misses rather than failures: its average top-$3$ accuracy reaches $0.99$. On Online Boutique, \myalgo{} again ranks first, at $0.78$, ahead of RCG ($0.71$), NSigma ($0.70$), BARO ($0.62$), CIRCA ($0.52$), and RCD ($0.49$), with an average top-$3$ accuracy of $0.90$; the margins over BARO, CIRCA, and RCD are significant, while those over NSigma and RCG are not at this sample size. Further per-benchmark observations are discussed in Appendix~\ref{app:micro}.
Notably, \myalgo{} attains these results without any causal structure, while the two graph-based baselines are given the true call graph. Together with the %
results in Section~\ref{exp:sim-dat}, this supports the view that the advantage of \myalgo{} does not depend on access to an accurate structure.

\section{Conclusion}
In this paper, we address the problem of root cause analysis in presence of any number of unobserved confounders.
For that purpose, we propose a sound approach that can
detect the root causes in parallel by searching for feasible structural causal models. Finally, we demonstrate our performance on synthetic dataset and real-world physical testbed. In our future work, we aim to relax the assumption on partial order and make it suitable for high-dimensional variables.

\subsection*{Acknowledgments}
This research has been supported in part by NSF CAREER 2239375, IIS 2348717, Amazon Research Award, Adobe Research and Intuit.

\bibliography{ref}

@article{kiciman2024causal,
  title   = {Causal Reasoning and Large Language Models: Opening a New Frontier for Causality},
  author  = {K{\i}c{\i}man, Emre and Ness, Robert and Sharma, Amit and Tan, Chenhao},
  journal = {Transactions on Machine Learning Research},
  year    = {2024}
}

@inproceedings{budhathoki2022causal,
  title={Causal structure-based root cause analysis of outliers},
  author={Budhathoki, Kailash and Minorics, Lenon and Bl{\"o}baum, Patrick and Janzing, Dominik},
  booktitle={International conference on machine learning},
  pages={2357--2369},
  year={2022},
  organization={PMLR}
}

@inproceedings{jiralerspong2024efficient,
  title     = {Efficient Causal Graph Discovery Using Large Language Models},
  author    = {Jiralerspong, Thomas and Chen, Xiaoyin and More, Yash and Shah, Vedant and Bengio, Yoshua},
  booktitle = {ICLR 2024 Workshop on How Far Are We From AGI},
  year      = {2024},
  note      = {arXiv:2402.01207}
}

@article{long2023causal,
  title   = {Causal Discovery with Language Models as Imperfect Experts},
  author  = {Long, Stephanie and Pich{\'e}, Alexandre and Zantedeschi, Valentina and Schuster, Tibor and Drouin, Alexandre},
  journal = {arXiv preprint arXiv:2307.02390},
  year    = {2023}
}

@article{evans2023latentfree,
  title   = {Latent-free equivalent {mDAGs}},
  author  = {Evans, Robin J.},
  journal = {Algebraic Statistics},
  volume  = {14},
  number  = {1},
  pages   = {3--16},
  year    = {2023}
}

@inproceedings{verma1990equivalence,
  title     = {Equivalence and synthesis of causal models},
  author    = {Verma, Thomas S. and Pearl, Judea},
  booktitle = {Proceedings of the Sixth Conference on Uncertainty in Artificial Intelligence (UAI)},
  pages     = {255--270},
  year      = {1990}
}

@article{richardson2023nested,
  title   = {Nested {M}arkov properties for acyclic directed mixed graphs},
  author  = {Richardson, Thomas S. and Evans, Robin J. and Robins, James M. and Shpitser, Ilya},
  journal = {The Annals of Statistics},
  volume  = {51},
  number  = {1},
  pages   = {334--361},
  year    = {2023}
}

@inproceedings{pearl1995testability,
  title     = {On the testability of causal models with latent and instrumental variables},
  author    = {Pearl, Judea},
  booktitle = {Proceedings of the Eleventh Conference on Uncertainty in Artificial Intelligence (UAI)},
  pages     = {435--443},
  year      = {1995}
}

@article{bell1964epr,
  title   = {On the {E}instein {P}odolsky {R}osen paradox},
  author  = {Bell, John S.},
  journal = {Physics Physique Fizika},
  volume  = {1},
  number  = {3},
  pages   = {195--200},
  year    = {1964}
}

@article{zhang2025agent,
  title={Which agent causes task failures and when? on automated failure attribution of llm multi-agent systems},
  author={Zhang, Shaokun and Yin, Ming and Zhang, Jieyu and Liu, Jiale and Han, Zhiguang and Zhang, Jingyang and Li, Beibin and Wang, Chi and Wang, Huazheng and Chen, Yiran and others},
  journal={arXiv preprint arXiv:2505.00212},
  year={2025}
}

@inproceedings{evans2012graphical,
  title={Graphical methods for inequality constraints in marginalized DAGs},
  author={Evans, Robin J},
  booktitle={2012 IEEE International Workshop on Machine Learning for Signal Processing},
  pages={1--6},
  year={2012},
  organization={IEEE}
}

@inproceedings{chen2014causeinfer,
  title={Causeinfer: Automatic and distributed performance diagnosis with hierarchical causality graph in large distributed systems},
  author={Chen, Pengfei and Qi, Yong and Zheng, Pengfei and Hou, Di},
  booktitle={IEEE INFOCOM 2014-IEEE Conference on Computer Communications},
  pages={1887--1895},
  year={2014},
  organization={IEEE}
}

@inproceedings{ma2019ms,
  title={Ms-rank: Multi-metric and self-adaptive root cause diagnosis for microservice applications},
  author={Ma, Meng and Lin, Weilan and Pan, Disheng and Wang, Ping},
  booktitle={2019 IEEE International Conference on Web Services (ICWS)},
  pages={60--67},
  year={2019},
  organization={IEEE}
}

@inproceedings{meng2020localizing,
  title={Localizing failure root causes in a microservice through causality inference},
  author={Meng, Yuan and Zhang, Shenglin and Sun, Yongqian and Zhang, Ruru and Hu, Zhilong and Zhang, Yiyin and Jia, Chenyang and Wang, Zhaogang and Pei, Dan},
  booktitle={2020 IEEE/ACM 28th International Symposium on Quality of Service (IWQoS)},
  pages={1--10},
  year={2020},
  organization={IEEE}
}

@inproceedings{wang2018cloudranger,
  title={Cloudranger: Root cause identification for cloud native systems},
  author={Wang, Ping and Xu, Jingmin and Ma, Meng and Lin, Weilan and Pan, Disheng and Wang, Yuan and Chen, Pengfei},
  booktitle={2018 18th IEEE/ACM International Symposium on Cluster, Cloud and Grid Computing (CCGRID)},
  pages={492--502},
  year={2018},
  organization={IEEE}
}

@inproceedings{li2022causal,
  title={Causal inference-based root cause analysis for online service systems with intervention recognition},
  author={Li, Mingjie and Li, Zeyan and Yin, Kanglin and Nie, Xiaohui and Zhang, Wenchi and Sui, Kaixin and Pei, Dan},
  booktitle={Proceedings of the 28th ACM SIGKDD Conference on Knowledge Discovery and Data Mining},
  pages={3230--3240},
  year={2022}
}

@article{ikram2022root,
  title={Root cause analysis of failures in microservices through causal discovery},
  author={Ikram, Azam and Chakraborty, Sarthak and Mitra, Subrata and Saini, Shiv and Bagchi, Saurabh and Kocaoglu, Murat},
  journal={Advances in Neural Information Processing Systems},
  volume={35},
  pages={31158--31170},
  year={2022}
}

@inproceedings{pham2024root,
  title={Root cause analysis for microservice system based on causal inference: How far are we?},
  author={Pham, Luan and Ha, Huong and Zhang, Hongyu},
  booktitle={Proceedings of the 39th IEEE/ACM International Conference on Automated Software Engineering},
  pages={706--715},
  year={2024}
}

@article{pham2024baro,
  title={Baro: Robust root cause analysis for microservices via multivariate bayesian online change point detection},
  author={Pham, Luan and Ha, Huong and Zhang, Hongyu},
  journal={Proceedings of the ACM on Software Engineering},
  volume={1},
  number={FSE},
  pages={2214--2237},
  year={2024},
  publisher={ACM New York, NY, USA}
}

@article{xin2023causalrca,
  title={Causalrca: Causal inference based precise fine-grained root cause localization for microservice applications},
  author={Xin, Ruyue and Chen, Peng and Zhao, Zhiming},
  journal={Journal of Systems and Software},
  volume={203},
  pages={111724},
  year={2023},
  publisher={Elsevier}
}

@inproceedings{kocaoglu2018causalgan,
  title={CausalGAN: Learning Causal Implicit Generative Models with Adversarial Training},
  author={Kocaoglu, Murat and Snyder, Christopher and Dimakis, Alexandros G and Vishwanath, Sriram},
  booktitle={International Conference on Learning Representations},
  year={2018}
}

@article{xia2021causal,
  title={The causal-neural connection: Expressiveness, learnability, and inference},
  author={Xia, Kevin and Lee, Kai-Zhan and Bengio, Yoshua and Bareinboim, Elias},
  journal={Advances in Neural Information Processing Systems},
  volume={34},
  pages={10823--10836},
  year={2021}
}

@inproceedings{
rahman2024modular,
title={Modular Learning of Deep Causal Generative Models for High-dimensional Causal Inference},
author={Md Musfiqur Rahman and Murat Kocaoglu},
booktitle={Forty-first International Conference on Machine Learning},
year={2024},
url={https://openreview.net/forum?id=bOhzU7NpTB}
}

@inproceedings{zheng2024mulan,
  title={MULAN: multi-modal causal structure learning and root cause analysis for microservice systems},
  author={Zheng, Lecheng and Chen, Zhengzhang and He, Jingrui and Chen, Haifeng},
  booktitle={Proceedings of the ACM Web Conference 2024},
  pages={4107--4116},
  year={2024}
}

@inproceedings{yao2024chain,
  title={Chain-of-event: Interpretable root cause analysis for microservices through automatically learning weighted event causal graph},
  author={Yao, Zhenhe and Pei, Changhua and Chen, Wenxiao and Wang, Hanzhang and Su, Liangfei and Jiang, Huai and Xie, Zhe and Nie, Xiaohui and Pei, Dan},
  booktitle={Companion Proceedings of the 32nd ACM International Conference on the Foundations of Software Engineering},
  pages={50--61},
  year={2024}
}

@article{okati2024root,
  title={Root cause analysis of outliers with missing structural knowledge},
  author={Okati, Nastaran and Hernan Garrido Mejia, Sergio and Orchard, William Roy and Bl{\"o}baum, Patrick and Janzing, Dominik},
  journal={arXiv e-prints},
  pages={arXiv--2406},
  year={2024}
}

@article{bai2024causal,
  title={The causal effects of global supply chain disruptions on macroeconomic outcomes: evidence and theory},
  author={Bai, Xiwen and Fern{\'a}ndez-Villaverde, Jes{\'u}s and Li, Yiliang and Zanetti, Francesco},
  year={2024},
  publisher={National Bureau of Economic Research}
}

@article{inoue2021new,
  title={A new approach to measuring economic policy shocks, with an application to conventional and unconventional monetary policy},
  author={Inoue, Atsushi and Rossi, Barbara},
  journal={Quantitative Economics},
  volume={12},
  number={4},
  pages={1085--1138},
  year={2021},
  publisher={Wiley Online Library}
}

@article{strobl2023identifying,
  title   = {Identifying Patient-Specific Root Causes with the Heteroscedastic Noise Model},
  author  = {Strobl, Eric V. and Lasko, Thomas A.},
  journal = {Journal of Computational Science},
  volume  = {72},
  pages   = {102099},
  year    = {2023}
}

@article{strobl2024counterfactual,
  title   = {Counterfactual Formulation of Patient-Specific Root Causes of Disease},
  author  = {Strobl, Eric V.},
  journal = {Journal of Biomedical Informatics},
  pages   = {104585},
  year    = {2024}
}

@article{vanhoeyveld2020vat,
  title   = {Value-Added Tax Fraud Detection with Scalable Anomaly Detection Techniques},
  author  = {Vanhoeyveld, Jellis and Martens, David and Peeters, Bruno},
  journal = {Applied Soft Computing},
  volume  = {86},
  pages   = {105895},
  year    = {2020}
}

@article{das2023algorithmic,
  title   = {Algorithmic Fairness},
  author  = {Das, Sanjiv and Stanton, Richard and Wallace, Nancy},
  journal = {Annual Review of Financial Economics},
  volume  = {15},
  pages   = {565--593},
  year    = {2023}
}

@article{wang2023hierarchical,
  title   = {Hierarchical Graph Neural Networks for Causal Discovery and Root Cause Localization},
  author  = {Wang, Dongjie and Chen, Zhengzhang and Ni, Jingchao and Tong, Liang and Wang, Zheng and Fu, Yanjie and Chen, Haifeng},
  journal = {arXiv preprint arXiv:2302.01987},
  year    = {2023}
}

@article{lin2024root,
  title   = {Root Cause Analysis in Microservice Using Neural Granger Causal Discovery},
  author  = {Lin, Cheng-Ming and Chang, Ching and Wang, Wei-Yao and Wang, Kuang-Da and Peng, Wen-Chih},
  journal = {arXiv preprint arXiv:2402.01140},
  year    = {2024}
}

@inproceedings{liu2021microhecl,
  title     = {Microhecl: High-Efficient Root Cause Localization in Large-Scale Microservice Systems},
  author    = {Liu, Dewei and He, Chuan and Peng, Xin and Lin, Fan and Zhang, Chenxi and Gong, Shengfang and Li, Ziang and Ou, Jiayu and Wu, Zheshun},
  booktitle = {Proceedings of the 43rd International Conference on Software Engineering: Software Engineering in Practice},
  series    = {ICSE-SEIP '21},
  pages     = {338--347},
  publisher = {IEEE Press},
  year      = {2021}
}

@inproceedings{shan2019epsilon,
  title     = {$\epsilon$-Diagnosis: Unsupervised and Real-Time Diagnosis of Small-Window Long-Tail Latency in Large-Scale Microservice Platforms},
  author    = {Shan, Huasong and Chen, Yuan and Liu, Haifeng and Zhang, Yunpeng and Xiao, Xiao and He, Xiaofeng and Li, Min and Ding, Wei},
  booktitle = {The World Wide Web Conference},
  pages     = {3215--3222},
  year      = {2019}
}

@inproceedings{ma2020automap,
  title     = {Automap: Diagnose Your Microservice-Based Web Applications Automatically},
  author    = {Ma, Meng and Xu, Jingmin and Wang, Yuan and Chen, Pengfei and Zhang, Zonghua and Wang, Ping},
  booktitle = {Proceedings of The Web Conference 2020},
  pages     = {246--258},
  year      = {2020}
}

@inproceedings{knorr1999finding,
  title     = {Finding Intensional Knowledge of Distance-Based Outliers},
  author    = {Knorr, Edwin M. and Ng, Raymond T.},
  booktitle = {VLDB},
  volume    = {99},
  pages     = {211--222},
  year      = {1999}
}

@inproceedings{micenkova2013explaining,
  title     = {Explaining Outliers by Subspace Separability},
  author    = {Micenkov{\'a}, Barbora and Ng, Raymond T. and Dang, Xuan-Hong and Assent, Ira},
  booktitle = {2013 IEEE 13th International Conference on Data Mining},
  pages     = {518--527},
  publisher = {IEEE},
  year      = {2013}
}

@article{liu2017contextual,
  title   = {Contextual Outlier Interpretation},
  author  = {Liu, Ninghao and Shin, Donghwa and Hu, Xia},
  journal = {arXiv preprint arXiv:1711.10589},
  year    = {2017}
}

@article{macha2018explaining,
  title   = {Explaining Anomalies in Groups with Characterizing Subspace Rules},
  author  = {Macha, Meghanath and Akoglu, Leman},
  journal = {Data Mining and Knowledge Discovery},
  volume  = {32},
  pages   = {1444--1480},
  year    = {2018}
}

@inproceedings{gupta2019beyond,
  title     = {Beyond Outlier Detection: LookOut for Pictorial Explanation},
  author    = {Gupta, Nikhil and Eswaran, Dhivya and Shah, Neil and Akoglu, Leman and Faloutsos, Christos},
  booktitle = {Machine Learning and Knowledge Discovery in Databases: European Conference, ECML PKDD 2018, Dublin, Ireland, September 10--14, 2018, Proceedings, Part I},
  pages     = {122--138},
  publisher = {Springer},
  year      = {2019}
}

@article{jaber2020causal,
  title={Causal discovery from soft interventions with unknown targets: Characterization and learning},
  author={Jaber, Amin and Kocaoglu, Murat and Shanmugam, Karthikeyan and Bareinboim, Elias},
  journal={Advances in neural information processing systems},
  volume={33},
  pages={9551--9561},
  year={2020}
}

@inproceedings{ikram2025root,
  title={Root cause analysis of failures from partial causal structures},
  author={Ikram, Azam and Lee, Kenneth and Agarwal, Shubham and Saini, Shiv Kumar and Bagchi, Saurabh and Kocaoglu, Murat},
  booktitle={The 41st Conference on Uncertainty in Artificial Intelligence},
  year={2025}
}

@book{pearl2009causality,
  title={Causality},
  author={Pearl, J.},
  isbn={9781139643986},
  url={https://books.google.com/books?id=LLkhAwAAQBAJ},
  year={2009},
  publisher={Cambridge University Press}
}

@article{gamella2025causal,
  title={Causal chambers as a real-world physical testbed for AI methodology},
  author={Gamella, Juan L and Peters, Jonas and B{\"u}hlmann, Peter},
  journal={Nature Machine Intelligence},
  volume={7},
  number={1},
  pages={107--118},
  year={2025},
  publisher={Nature Publishing Group UK London}
}

@inproceedings{ikram2022rcd,
  author    = {Ikram, Azam and Chakraborty, Sarthak and Mitra, Subrata and Saini, Shiv Kumar and Bagchi, Saurabh and Kocaoglu, Murat},
  title     = {Root Cause Analysis of Failures in Microservices through Causal Discovery},
  booktitle = {Advances in Neural Information Processing Systems},
  volume    = {35},
  year      = {2022}
}

@inproceedings{li2022circa,
  author    = {Li, Mingjie and Li, Zeyan and Yin, Kanglin and Nie, Xiaohui and Zhang, Wenchi and Sui, Kaixin and Pei, Dan},
  title     = {Causal Inference-Based Root Cause Analysis for Online Service Systems with Intervention Recognition},
  booktitle = {Proceedings of the 28th ACM SIGKDD Conference on Knowledge Discovery and Data Mining},
  pages     = {3230--3240},
  year      = {2022},
  doi       = {10.1145/3534678.3539041}
}

@inproceedings{ikram2025rcg,
  author    = {Ikram, Azam and Lee, Kenneth and Agarwal, Shubham and Saini, Shiv Kumar and Bagchi, Saurabh and Kocaoglu, Murat},
  title     = {Root Cause Analysis of Failures from Partial Causal Structures},
  booktitle = {Proceedings of the Forty-first Conference on Uncertainty in Artificial Intelligence},
  series    = {Proceedings of Machine Learning Research},
  volume    = {286},
  pages     = {1794--1818},
  publisher = {PMLR},
  year      = {2025}
}

@inproceedings{orchard2025missingstructural,
  author    = {Orchard, William Roy and Okati, Nastaran and Garrido Mejia, Sergio Hernan and Bl{\"o}baum, Patrick and Janzing, Dominik},
  title     = {Root Cause Analysis of Outliers with Missing Structural Knowledge},
  booktitle = {Advances in Neural Information Processing Systems},
  volume    = {38},
  year      = {2025}
}

@misc{li2024cholesky,
  author        = {Li, Jinzhou and Chu, Benjamin B. and Scheller, Ines F. and Gagneur, Julien and Maathuis, Marloes H.},
  title         = {Root Cause Discovery via Permutations and Cholesky Decomposition},
  year          = {2024},
  eprint        = {2410.12151},
  archivePrefix = {arXiv},
  primaryClass  = {stat.ML}
}

@inproceedings{wu2021microdiag,
  author    = {Wu, Li and Tordsson, Johan and Bogatinovski, Jasmin and Elmroth, Erik and Kao, Odej},
  title     = {{MicroDiag}: Fine-Grained Performance Diagnosis for Microservice Systems},
  booktitle = {2021 IEEE/ACM International Workshop on Cloud Intelligence},
  pages     = {31--36},
  year      = {2021}
}

@inproceedings{yang2024learning,
  title={Learning unknown intervention targets in structural causal models from heterogeneous data},
  author={Yang, Yuqin and Salehkaleybar, Saber and Kiyavash, Negar},
  booktitle={International Conference on Artificial Intelligence and Statistics},
  pages={3187--3195},
  year={2024},
  organization={PMLR}
}

@article{javaloy2023causal,
  title={Causal normalizing flows: from theory to practice},
  author={Javaloy, Adri{\'a}n and S{\'a}nchez-Mart{\'\i}n, Pablo and Valera, Isabel},
  journal={Advances in Neural Information Processing Systems},
  volume={36},
  pages={58833--58864},
  year={2023}
}

@article{ansanelli2025observational,
  title   = {The observational partial order of causal structures with latent variables},
  author  = {Ansanelli, Marina Maciel and Wolfe, Elie and Spekkens, Robert W.},
  journal = {Journal of Causal Inference},
  volume  = {14},
  number  = {1},
  pages   = {20250009},
  year    = {2026},
  doi     = {10.1515/jci-2025-0009}
}

@inproceedings{pham2025rcaeval,
  title={RCAEval: A Benchmark for Root Cause Analysis of Microservice Systems with Telemetry Data},
  author={Pham, Luan and Zhang, Hongyu and Ha, Huong and Salim, Flora and Zhang, Xiuzhen},
  booktitle={Companion Proceedings of the ACM on Web Conference 2025},
  pages={777--780},
  year={2025}
}
\bibliographystyle{preprint}

\appendix

\clearpage

\section{Additional Discussion}
\subsection{Limitations and Future works}
\label{sec:limit}
We assume we have access to a causal structure or its total order. Training neural networks might be resource consuming and time consuming for some applications. We assume that normal and anomalous dataset is clearly separated where anomaly point detect might be challenging for some tasks. Due to convergence error in neural networks, we might always have non-zero value of the loss function even if there is not mechanism shift. 
Although we showed analysis for denser and sparser graph, in this work, we did not explore with the mis-specified graph is arbitrarily different.
We aim to address these limitations in our future work.

\subsection{Broader Impact}
\label{sec:broader-impact}
Root cause analysis in the presence of unobserved confounders can improve the reliability and safety of complex systems by helping practitioners identify failure sources even when some relevant variables are hidden or unmeasured. By using deep causal models, our approach can support more accurate diagnosis in domains such as cloud computing, healthcare, manufacturing, and critical infrastructure, where failures may have significant economic or societal consequences. At the same time, incorrect causal assumptions or misspecified graphs may lead to misleading root cause conclusions, especially in high-stakes applications. Therefore, our method should be used with domain knowledge, uncertainty assessment, and careful validation before deployment in real-world decision-making systems.

\section{Extended Related Work}\label{app:related}

\paragraph{Causal graph-based RCA.}
RCA is an important problem in microservice systems, which have increasingly adopted causal inference as a principled approach. Early causal RCA methods construct a causal graph over service-level metrics and identify the root cause by traversing or ranking nodes with an anomaly score~\citep{chen2014causeinfer}. Many approaches estimate the causal structure from observational data using constraint-based discovery algorithms such as PC, and then apply scoring methods to traverse the graph and localize the root cause~\citep{wang2018cloudranger, ma2019ms, meng2020localizing}. MicroDiag~\citep{wu2021microdiag} performs fine-grained diagnosis using metric-level dependency graphs, but it remains limited to the observed metrics and does not address unobserved confounding. A large-scale evaluation of causal discovery-based RCA methods further shows that no existing method performs robustly across datasets and failure types~\citep{pham2024root}.

\paragraph{RCA as intervention recognition.}
Several recent works model system failures as interventions on causal mechanisms, building on the soft-intervention view of distribution shifts~\citep{jaber2020causal}. RCD~\citep{ikram2022root} treats a failure as a soft intervention and performs local causal discovery to avoid learning the full graph, and RCG~\citep{ikram2025root} extends this idea using partial causal structures. CIRCA~\citep{li2022circa} formulates RCA as an intervention recognition problem, based on a sufficient condition for a variable in a causal Bayesian network to be a root cause. CausalRCA~\citep{xin2023causalrca} learns the causal graph with a gradient-based causal discovery method. These methods depend either on expert-provided or learned graphs, which may be inaccurate under noise and hidden confounding, or on conditional independence and invariance tests, whose reliability degrades with continuous variables and growing conditioning sets.

\paragraph{RCA with limited structural knowledge.}
\citet{orchard2025missingstructural} formalize RCA as identifying the mechanism change under a single-target soft intervention, provide theoretical guarantees for restricted graph classes such as polytrees, and propose the efficient SO and ST methods for settings with missing structural knowledge. Cholesky-based RCA~\citep{li2024cholesky} exploits invariances specific to linear SCMs. Neither approach explicitly models nonlinear mechanisms with latent confounding.

\paragraph{Counterfactual and score-based formulations.}
\citet{budhathoki2022causal} %
formalize RCA as a quantitative contribution analysis based on counterfactuals, attributing an anomaly to the mechanisms whose deviation from normal behavior explains it. Their approach uses a graph learned from normal operation and allows anomalous samples from multiple distributions, but it requires knowing the full SCM, including invertible functional relations, which are hard to estimate in practice. \citet{okati2024root} propose a score function for RCA that requires no causal knowledge.

\paragraph{Statistical and multi-modal methods.}
BARO~\citep{pham2024baro} uses multivariate change point detection to model the dependency structure of multivariate time series and provides robust statistical scoring, but it does not use causal structure and, like other statistical testing methods, implicitly assumes that root causes exhibit the largest distributional shifts. Microservice data are also inherently multi-modal, comprising metrics, logs, and traces. \citet{zheng2024mulan} propose a unified multi-modal causal structure learning framework that combines language-model log representations with contrastive learning to align modality-invariant and modality-specific information, and \citet{yao2024chain} construct weighted event causal graphs from multi-modal observations to improve interpretability and alignment with domain knowledge.

\paragraph{Positioning of our work.}
In contrast to the above, \myalgo assumes neither causal sufficiency nor linearity and does not rely solely on conditional independence or invariance tests. It requires only a causal graph consistent with the true partial order and uses deep causal models to search for SCMs consistent with both the normal and anomalous datasets, thereby exploiting general distributional constraints in the presence of an arbitrary number of unobserved confounders.

\section{Assumptions} 
\label{sec:assumptions}
In this section, we enlist all assumptions considered in our paper. We also repeat them at appropriate places when needed. 

\begin{assumption}[Semi-Markovian SCM]\label{asm:scm}
    The data are generated by an acyclic semi-Markovian SCM; each unobserved confounder
    affects two observed variables, shown as $V_i \leftrightarrow V_j$ in the ADMG.
\end{assumption}

\begin{assumption}[Mechanism shift]\label{asm:shift}
    $\mathcal{M}_n$ and $\mathcal{M}_a$ share the ADMG $G^*$ and $P(\mathbf{U})$, and
    $f^n_V = f^a_V$ for all $V \notin R^*$. Samples are labeled by regime,
    $\mathcal{D}_n \sim P_n(\mathbf{V})$ and $\mathcal{D}_a \sim P_a(\mathbf{V})$.
\end{assumption}
An anomaly changes only the mechanisms of $R^*$, not the graph or the latent distribution.

\begin{assumption}[Graph access]\label{asm:graph}
    We are given the true ADMG $G^*$ or partial knowledge about it.
\end{assumption}
Section~4.3 relaxes this to any $G' \supseteq G^*$, which can be built from the true partial order alone.

\begin{assumption}[Expressive DCM]\label{asm:dcm}
    The DCM is representative enough for the considered ADMG.
\end{assumption}

\begin{assumption}[Identifiability]\label{asm:id}
    $\texttt{SolRC}(G^*) = \{R : R^* \subseteq R \subseteq \mathbf{V}\}$.
\end{assumption}

\begin{assumption}[Latent invariance]\label{ass:latent-inv}
The normal and anomalous SCMs share the same exogenous distribution $P(\mathbf U)$; only the
structural functions of the root causes $R^*$ differ between $\mathcal M_n^*$ and $\mathcal M_a^*$.
\end{assumption}

\begin{assumption}[{Root-cause identifiability on the true graph}]\label{ass:ident}
$\PossRC(\Gtrue)$ has a unique inclusion-minimal element.necessarily $\Rtrue$.

\end{assumption}

\section{Possible Root Cause set}
\label{sec:proof-solrc}

\lemupclosure*

\begin{proof}
Let $P$ realize $(\Pn, \Pa)$ on $G \cup F(R)$. Build $P'$ on $G \cup F(R')$ with
the same mechanisms, except that each newly targeted node $V_i \in R' \setminus
R$ is given the $F$-augmented mechanism that ignores $F$ and equals its
mechanism in $P$. The conditional distributions in both regimes are unchanged,
so $P'$ realizes $(\Pn, \Pa)$ and lies in $\M\bigl(G \cup F(R')\bigr)$.
\end{proof}

\begin{lemma}[{When is the root-cause number finite}]\label{lem:finite}

Let $G^*$ and $G'$ be the true and mis-specified graphs,  either by adding or by removing edges. Given, $(\Pn, \Pa) \text{ realized by } (G^*, R^*)$, we have
$\PossRC(G') \neq \emptyset$ if and only if $\Pn \in \M(G')$ and $\Pa \in \M(G')$.

\end{lemma}

\begin{proof}

(If) $\PossRC(G') \neq \emptyset$ implies that there exists some  mechanism  set $R$ that can be shifted to change from normal distribution $P_n$ to anomalous distribution $P_a$.
By
Lemma~\ref{lem:upclosure},   $R\subseteq  \V \text{(the full set)} \in \PossRC(G')$ . 
Thus, $F=0$ with $\{f^n_{V}\}_{\forall V\in \V}$, gives a distribution $\Pn$ generated by $G'$, hence $\Pn \in \M(G')$. 
Similarly, 
$F=1$ with $\{f^a_{V}\}_{\forall V\in \V}$, gives a distribution $\Pa$ generated by $G'$, hence $\Pa \in \M(G')$. Precisely, whether the mis-specified graph $G'$ is sparser or denser compared to $G^*$, $G'$ can realize $P_n$, and by shifting all mechanisms: $\{f^n_{V} \rightarrow f^a_{V}\}_{\forall V\in \V}$, can change $P_n$ to $P_a$.

(Only If)
Conversely, suppose $\Pn, \Pa \in \M(G')$: the normal and anomalous distribution generated by the true graph $G^*$ is realizable by the mis-specified graph $G'$.
We need to prove $\PossRC(G')\neq \emptyset$.
Realize $\Pn$ by $G$ with some latent
distribution $\nu_n$ and mechanisms $f^n$, and $\Pa$ with latent distribution
$\nu_a$ and mechanisms $f^a$. 

We can construct a causal model by taking
the shared
latent to be the independent product $L = (L^n, L^a) \sim \nu_n \times \nu_a$.
Let use consider every node as a root cause ($R = \V$).

When $F=0$, let each visible mechanism read the $L^n$-component and apply $f^n$ for all $\{f_{V\in \mathbf{V}}\}$, 
and  when $F=1$, read the $L^a$-component and apply $f^a$, for all $\{f_{V\in \mathbf{V}}\}$.

This is a valid element of $\M\bigl(G \cup F(\V)\bigr)$ and realizes $(\Pn, \Pa)$, so
$\V \in \PossRC(G)$.
\end{proof}

\begin{corollary}
    $\rc(G) < \infty$ if and only if $\Pn \in \M(G)$ and $\Pa \in \M(G)$.
\end{corollary}

\lemltosolrc*

\begin{proof}
We first show that $\mathbf V\setminus\{Y\} \in \PossRC(G^*)$ implies $\ell^*(G^*, \mathbf V\setminus\{Y\}) = 0$.
Let $P \in \M\bigl(G^* \cup F(\mathbf V\setminus\{Y\})\bigr)$ realize $(P_n^*, P_a^*)$, with structural equations
$V_i = g_i(\mathrm{pa}_{G^*}(V_i), F, U_i)$, where $g_i$ depends on $F$ only if $V_i \neq Y$, and the latents
$\mathbf U$ are independent of $F$ because $F$ is a root. Setting
$\widehat f^{\,n}_i(\cdot) = g_i(\cdot, 0, \cdot)$ and $\widehat f^{\,a}_i(\cdot) = g_i(\cdot, 1, \cdot)$ with the common
latent distribution $P(\mathbf U)$ gives an SCM pair on $G^*$ with $\widehat f^{\,n}_Y = \widehat f^{\,a}_Y$, i.e., an element
of $\widehat M_{na}(G^*, \mathbf V\setminus\{Y\})$. Since $\mathbf U \indep F$, it induces
$P(\V \mid F{=}0) = P_n^*$ and $P(\V \mid F{=}1) = P_a^*$, so both discrepancies vanish and
$\ell^*(G^*, \mathbf V\setminus\{Y\}) = 0$.

By contraposition, $\ell^*(G^*, \mathbf V\setminus\{Y\}) > 0$ implies $\mathbf V\setminus\{Y\} \notin \PossRC(G^*)$.
Finally, for any $\mathbf Z \subseteq \mathbf V$, if $\mathbf Z\setminus\{Y\} \in \PossRC(G^*)$, then, since
$\mathbf Z\setminus\{Y\} \subseteq \mathbf V\setminus\{Y\}$, Lemma~\ref{lem:upclosure} would give
$\mathbf V\setminus\{Y\} \in \PossRC(G^*)$, a contradiction. Hence $\mathbf Z\setminus\{Y\} \notin \PossRC(G^*)$.
\end{proof}

\thmrcsound*

\begin{proof}
By latent invariance, the true pair $(\mathcal M_n^*, \mathcal M_a^*)$ shares $P(\mathbf U)$ and differs only in the
mechanisms of $R^*$, so it realizes $(P_n^*, P_a^*)$ on $G^* \cup F(R^*)$ (each $V_i \in R^*$ uses its normal or
anomalous mechanism according to $F$, and every other mechanism ignores $F$). Hence $R^* \in \PossRC(G^*)$.
Applying Lemma~\ref{lemm-ltoSolrc} with $\mathbf Z = R^*$ gives $R^*\setminus\{Y\} \notin \PossRC(G^*)$. Therefore
$R^*\setminus\{Y\} \neq R^*$, i.e., $Y \in R^*$.
\end{proof}

\propsoundfree*

\begin{proof}
The true pair $(\M_n^*, \M_a^*)$ realizes $(\Pn, \Pa)$ while holding every
mechanism outside $\Rtrue$ invariant, so $\Rtrue \in \PossRC(\Gtrue)$. The
intersection of a family is contained in each of its members.
\end{proof}

\begin{restatable}{assumption}{assmainident}
\label{ass:main-ident}
$\PossRC(\Gtrue)$ has a unique inclusion-minimal element.
\end{restatable}

\begin{restatable}{corollary}{corcompleteident}
\label{cor:complete-ident}
Under Assumption~\ref{ass:main-ident} and  Assumption~\ref{asm:dcm},
$
  \widehat{R} \;=\; \bigcap \PossRC(\Gtrue)
  \;=\; \argmin_{R \in \PossRC(\Gtrue)} |R|
  \;\subseteq\; \Rtrue .
$
\end{restatable}

\begin{proof}
Let $m$ be the unique inclusion-minimal element of $\PossRC(\Gtrue)$ (Assumption~\ref{ass:main-ident}).
Since $\PossRC(\Gtrue) \subseteq 2^{\V}$ is finite, every $R \in \PossRC(\Gtrue)$ contains an
inclusion-minimal element, which must be $m$; hence $m \subseteq R$ for all $R \in \PossRC(\Gtrue)$.
Therefore $\bigcap \PossRC(\Gtrue) = m$, and $|R| \ge |m|$ for every member, with equality only if
$R = m$, so $m$ is the unique minimizer of $|R|$. Finally, $\Rtrue \in \PossRC(\Gtrue)$
(Proposition~\ref{prop:sound-free}), so $m \subseteq \Rtrue$.
\end{proof}

\section{Error analysis in RCA-DCM ranking}
\label{appex:erro-anal}

Note, Theorem~\ref{th-main:rc_sound} says that if $\ell(G^*, \mathbf{V}\setminus \{Y\} ) > 0 $ then $Y$ is a root cause. However, in practice due to low sample size (with perfect model training), in all cases we find $\ell(.)>0$. Here, we show that under a condition, in such cases, $\myalgo$ ranks the true root cause at the top.
For any candidate $Y \in \mathbf{V}$ and let
$s(Y) \;=\; \ell_m(Y) - \ell^*(Y)$
denote the \emph{statistical error} originated due to a low sample size since the loss calculated might not represent the
population value. This term is two-sided: the finite set of samples might
produce a lower loss,
$\ell_m(Y) \in [\ell(Y) - \varepsilon_{\mathrm{stat}},\,
\ell^*(Y)]$ or a higher loss
$\ell_m(Y) \in [\ell^*(Y),\,
\ell(Y) + \varepsilon_{\mathrm{stat}}]$. We assume the empirical loss concentrates around its population value uniformly
over candidates: for every $\delta \in (0,1)$ there is an
$\varepsilon_{\mathrm{stat}} > 0$, decreasing in the sample size, such that the
event
$E = \{\, |s(Y)| \le \varepsilon_{\mathrm{stat}} \;\; \forall Y \in \mathbf{V}
\,\}$ satisfies $\Pr(E) \ge 1-\delta$.

\begin{proposition}[Correctness of Ranking]
\label{prop:topk}
Suppose our model training is perfect, if $$\min_{Y \in R^*} \ell^*(Y) > 2\varepsilon_{\mathrm{stat}} $$ then, with probability at least $1-\delta$,
$$
\min_{Y \in R^*} \hat\ell(Y) \;>\; \max_{Z \notin R^*} \hat\ell(Z),
$$
\end{proposition}
Intuitively, every root cause is scored strictly above every non-root cause,
placing the $k = |R^*|$ root causes
in its top $k$ positions in the ranking returned by \myalgo.

\begin{proof}

We can construct a bound for the empirical loss as below.    

\begin{equation}
    \ell^*(Y) - \varepsilon_{\mathrm{stat}} \leq  \hat\ell(Y) \leq 
    \ell^*(Y) + \varepsilon_{\mathrm{stat}}
\end{equation}

Consider a non-root cause variable $Z$ and its loss such that 
\begin{equation}
\label{eq:max-rc}
\begin{split}
    Z&=argmax_{V\notin R^*}  \hat\ell(V)\\
     \hat\ell(Z) &\leq \ell^*(Z) +\varepsilon_{\mathrm{stat}} \\
    \hat\ell(Z) &\leq \varepsilon_{\mathrm{stat}} 
\end{split}
\end{equation}
Here, $\ell^*(Z)=0$ for non-root cause variables.

Similarly, we consider a root cause variable $Y$ and its loss such that 
\begin{equation}
\label{eq:min-rc}
\begin{split}
Y &= min_{Y\in R^*}  \hat\ell(Y)\\
\hat\ell(Y) &\geq \ell^*(Y) - \varepsilon_{\mathrm{stat}}
\end{split}
\end{equation}

We subtract $ \hat\ell(Z)$ from both sides of Equation~\ref{eq:min-rc}.
\begin{equation}
\begin{split}
    \hat\ell(Y)  -\hat\ell(Z) & \geq \ell^*(Y) - \varepsilon_{\mathrm{stat}} -  \hat\ell(Z)\\
    \hat\ell(Y)  -\hat\ell(Z) & \geq 
\ell^*(Y) - \varepsilon_{\mathrm{stat}}
    -\varepsilon_{\mathrm{stat}} \\
    \hat\ell(Y)  -\hat\ell(Z) & \geq 
\ell^*(Y)
     -2\varepsilon_{\mathrm{stat}} 
\end{split}
\end{equation}

Thus, if $\ell^*(Y) 
     -2\varepsilon_{\mathrm{stat}}  >0 \implies  \ell^*(Y)
     >2\varepsilon_{\mathrm{stat}}   $
then $\hat\ell(Y)  -  \hat\ell(Z)>0$ which implies that we will obtain the $k$ root causes in the top $k$ of the output ranking. 
\end{proof}

\section{Graph mis-specification}
\label{sec:proof-mis}

\graphmono*

\begin{proof}
Fix any $R \in \PossRC(G_1)$.
According to Lemma~\ref{lem:finite}, $R \in \PossRC(G_1)$ implies that $P_n = P(\V \mid F=0) \in \M(G_1)$ and $P_a = P(\V \mid F=1) \in \M(G_1)$. Since we can shift $P_n \rightarrow P_a$ with root causes $R$,
we have $P(\V,F) \in \M\bigl(G_1 \cup F(R)\bigr)$.

The augmented graphs $G_1 \cup F(R)$ and $G_2 \cup F(R)$ have
identical $F$-edges, and $G_1 \subseteq G_2$ gives directed-edge and
simplicial-complex inclusion, so $G_2 \cup F(R)$ structurally dominates
$G_1 \cup F(R)$ in the sense of Definition~\ref{def:str-dom}. By Lemma~\ref{lem:struct-dom-obs-dom} (structural
dominance implies observational dominance),
\[
\M\bigl(G_1 \cup F(R)\bigr) \subseteq \M\bigl(G_2 \cup F(R)\bigr)
\]
at every cardinality, in particular at the cardinalities of the data.

Thus
$P(\V,F)$ belongs to $\M\bigl(G_2 \cup F(R)\bigr)$ (right side) as well.
This implies that $P_n = P(\V \mid F=0) \in \M(G_2)$ and $P_a = P(\V \mid F=1) \in \M(G_2)$, and the shift $P_n \rightarrow P_a$ is possible with root causes $R$ in the causal model constructed with the denser graph $G_2$.
Therefore, according to Lemma~\ref{lem:finite},
$R \in \PossRC(G_2)$.

\end{proof}

\master*

\begin{proof}
By Lemma~\ref{lem:main-graphmono}, $\PossRC(G_1) \subseteq \PossRC(G_2)$. Minimizing
$|R|$ over the larger family $\PossRC(G_2)$ can only decrease the minimum:
$\rc(G_2, P_n, P_a) = \min_{R \in \PossRC(G_2)} |R| \le \min_{R \in \PossRC(G_1)} |R| = \rc(G_1, P_n, P_a)$.
(The inequality holds in $\{0,\dots,n\} \cup \{\infty\}$, with the convention
$\min \emptyset = \infty$.)
\end{proof}

\begin{theorem}[{root-cause number under graph misspecification} ]\label{thm:main}
Let the observed tuple $(\Pn, \Pa)$ be generated by the true model
$(\Gtrue, \Rtrue)$, so $\rc(\Gtrue) \le |\Rtrue| < \infty$.
\begin{enumerate}
\item[(i)] \textbf{(Over-specification is parsimony-safe.)} For any
mis-specified graph $G' \supseteq \Gtrue$,
$\rc(G') \le \rc(\Gtrue)$.
\item[(ii)] \textbf{(Under-specification is conservative.)} For any
mis-specified graph $G' \subseteq \Gtrue$,
$\rc(G') \ge \rc(\Gtrue)$. 
Moreover,  if $\rc(G') < \infty$, then 
 $\forall R \in \PossRC(G')$,
 we have $|R| \geq \rc(\Gtrue)$ .
\end{enumerate}
\end{theorem}

\begin{proof}
Both parts are instances of Lemma~\ref{lem:master}. For (i), take
$ G_D= G'   \supseteq G_S = \Gtrue  $ to get $\rc(G') \le \rc(\Gtrue)$. 

For (ii),
take $ G_S= G'  \subseteq G_D = \Gtrue$ to get $\rc(G') \ge  \rc(\Gtrue)$. 
Therefore,
$\min_{\PossRC(G')} |R| \ge \rc(\Gtrue) \implies \forall R \in \PossRC(G'),
|R| \geq \rc(\Gtrue)$
\end{proof}

\begin{proposition}[{Indispensable root causes, monotone}]\label{prop:indispensable}
If $G_S \subseteq G_D$ then
$\bigcap \PossRC(G_D) \subseteq \bigcap \PossRC(G_S)$.
\end{proposition}

\begin{proof}
By Lemma~\ref{lem:main-graphmono}, $\PossRC(G_S) \subseteq \PossRC(G_D)$. The
intersection of a smaller family is larger:
$\bigcap \PossRC(G_D) \subseteq \bigcap \PossRC(G_S)$.
\end{proof}

\begin{corollary}[{Faithful inclusion}]\label{cor:inclusion}
Under Assumption~\ref{ass:ident}, for any $G_S \subseteq \Gtrue$ with
$\{\Pn, \Pa\} \subseteq \M(G_S)$, every inclusion-minimal element of
$\PossRC(G_S)$ contains $\Rtrue$. That is, the root-cause set reported on the
sparser graph is a superset of the true root causes:
\[
\widehat{R} \supseteq  \Rtrue 
\quad\text{for every minimal } \widehat{R} \in \PossRC(G_S).
\]
\end{corollary}

\begin{proof}
A finite upward-closed family with a unique minimal element $m$ consists of all
supersets of $m$, so its intersection is $m$. By Assumption~\ref{ass:ident},
$\bigcap \PossRC(\Gtrue) = \Rtrue$. By Proposition~\ref{prop:indispensable},
$\Rtrue = \bigcap \PossRC(\Gtrue) \subseteq \bigcap \PossRC(G_S)$. The
intersection is contained in every member of $\PossRC(G_S)$, in particular in
every minimal one.
\end{proof}

\begin{corollary}
Under Assumption~\ref{ass:ident}, with $\{\Pn, \Pa\} \subseteq \M(G')$:
\begin{enumerate}
    \item[(i)] for any $G' \supseteq \Gtrue$, there exists an
    inclusion-minimal $\widehat{R} \in \PossRC(G')$ with
    $\widehat{R} \subseteq \Rtrue$;
    \item[(ii)] for any $G' \subseteq \Gtrue$, every inclusion-minimal
    $\widehat{R} \in \PossRC(G')$ satisfies $\widehat{R} \supseteq \Rtrue$.
\end{enumerate}
\end{corollary}

\begin{proof}
     i) By Assumption~\ref{ass:ident},
$\bigcap \PossRC(\Gtrue) = \Rtrue$.
Since $G_D \supseteq \Gtrue$, by Proposition~\ref{prop:indispensable},
$ \bigcap \PossRC(G_D)   \subseteq \bigcap \PossRC(\Gtrue) = \Rtrue  $. 

ii) By Lemma~\ref{lem:main-graphmono}, $\Rtrue \in \PossRC(\Gtrue) \subseteq \PossRC(G_D)$, so
$\Rtrue \in \PossRC(G_D)$. The family
$\mathcal F_{\Rtrue} := \{S \in \PossRC(G_D) : S \subseteq \Rtrue\}$
is nonempty and finite, hence has a $\subseteq$-minimal element
$\widehat R$. If some $S \in \PossRC(G_D)$ satisfied $S \subsetneq
\widehat R$, then $S \subseteq \Rtrue$ as well, contradicting minimality
of $\widehat R$ in $\mathcal F_{\Rtrue}$. So $\widehat R$ is
inclusion-minimal in all of $\PossRC(G_D)$.
\end{proof}

By Lemma~\ref{lem:main-graphmono} and ~\ref{lem:master}, if the mis-specified graph $G'$ is denser compared to the true graph $G^*$, i.e. $G_1=G^*, G_2= G'$ then the predicted root cause set will be a subset of the original root cause set. Whereas, if $G'$ is sparse compared to the true graph $G^*$, the predicted root cause set will be a superset of the original root cause set.  We formalize this below.

\section{{Experiment Details}}
\label{appex:exp}

\subsection{Training Details}
\label{app:train}

For all continuous experiments we use a flow-based DCM
(\texttt{dcm\_flow}), in which each node is modeled as a conditional
normalizing flow given its parents (and its latents, when present) and
trained by maximum likelihood. For each candidate we fit two DCMs, one on
the normal and one on the anomalous data. Under the shared-update mode used
in all our runs we optimize
\[
L = L_{\text{anomalous}} + L_{\text{normal}},
\]
with equal (unweighted) terms, matching Algorithm~\ref{alg:dcm-rca}. The
candidate mechanism may be shared across the two models
(\texttt{trn.upd=shared}), so that only the non-shared components absorb
regime-specific change. All models are trained with Adam at a learning rate
of $5\times10^{-4}$, batch size $256$, for $50$ epochs, using a hidden
dimension of $32$ and a latent noise dimension of $10$. Normal and
anomalous features are jointly normalized before training, and candidate
node tests are executed in parallel with up to $10$ workers
(\texttt{thr\_num}). After training, we rank candidates by a discrepancy
between model samples and observed data; primarily the Wasserstein
distance for continuous experiments, with MMD and median shift additionally
logged as diagnostics;and take the top-$k$ set for PRR. Full defaults are
given in \texttt{dcm/conf.yaml} and are overridable via Hydra; our code
repository (\url{https://github.com/Musfiqshohan/RCA-DCM}) contains the training code (\texttt{dcm/train.py})
along with configs, training scripts, and experiment runners. 

We use the following metrics for evaluating output rankings.
\begin{equation}
T@k \;=\; \frac{1}{|\mathbf{E}|}\sum_{e \in \mathbf{E}}
\frac{\left|\hat{R}_e[1,..,k] \cap R^{*}_{e}\right|}
{\min\!\left(k,\, |R^{*}_{e}|\right)},
\qquad
\mathrm{PRR} \;=\; \frac{1}{|\mathbf{E}|}\sum_{e \in \mathbf{E}}
\mathbf{1}\!\left[\,T@|R^{*}_{e}| = 1\,\right]
\label{eq:metrics}
\end{equation}

\textbf{Causal graph:} The causal structural knowledge can be obtained from domain experts or, as recent work shows, through \textit{inexpensive} LLM-based causal discovery~\cite{kiciman2024causal,jiralerspong2024efficient,long2023causal}. In real-world setups such as microservice and agentic frameworks, the call graph or system architecture can also be extracted directly. Thus, we do not consider this requirement a major limitation. Nonetheless, \myalgo{} outperforms the baselines in synthetic experiments with both the true graph and a dense graph, and in microservice experiments where no graph is known and we use a sink graph (all variables point to the target node).

\textbf{Sample size:} We evaluate \myalgo across a range of sample sizes, and it performs consistently well: in synthetic experiments ($N=5$k), in the Causal Chamber ($N=1$k), and in microservice setups ($N \sim 700$). More extensive experiments are in Appendix~\ref{sec:sample-sensitivity}.

\textbf{Runtime:} Our proposed method is computationally more expensive than some baselines (e.g., a single root-cause test takes 0.4\,s for RCD vs.\ 41\,s for ours). However, it outperforms them in almost all cases. Thus, our approach targets applications that prioritize reliability under latent confounding over faster runtime.

\subsection{Simulated Datasets}
\label{app:sim-dat}

We follow~\citep{yang2024learning} to generate synthetic normal
and anomalous data from randomly generated graphs, varying the total
variable count, latent proportion, and edge density. We sample a random DAG
over $p$ observed variables plus $m$ latent confounders,  observed-to-observed edges follow a random ancestral density, and
each latent connects to a small set of observed children.
Only observed variables are written to the dataset and latents are never revealed to the
algorithms. The causal mechanism of each variable combines linear structural mixing with a per-variable invertible non-linearity. Given exogenous noises, observations are produced by the
nonlinear SEM whose graph and mixing structure are shared across both
regimes.

In these experiments, we average results over 100 trials and for each trial  we draw $5000$ \emph{normal} samples with
baseline noise scales drawn uniformly from $[a_{\min}, a_{\max}]$ and no
interventions, and $5000$ \emph{anomalous} samples in which a subset of
\emph{observed} variables is selected as root causes; for those targets
only, the noise scale is redrawn from a stronger range
$[b_{\min}, b_{\max}]$, while non-targets retain their baseline scales.
Sample sizes vary across experiments and are stated per experiment. Root
causes are therefore noise-scale mechanism shifts, and the intervened
observed nodes are the ground-truth root causes. Intervention strengths can
optionally be made heterogeneous via a per-variable weight schedule.
Methods see
only the observed tables and, where applicable, the graph.

Independent exogenous noise $\varepsilon\in\mathbb{R}^{n_L+n_X}$
concatenates latent exogenous shocks $\ell\in\mathbb{R}^{n_L}$ and observed
exogenous disturbances $\varepsilon^{\mathsf X}\in\mathbb{R}^{n_X}$. A
sparse mask $G$ fixes which latent/observed parents may influence each
observed sensor; we draw a random column-$\ell_2$-normalized coefficient
matrix, apply $G$, and split the observed rows as
$\bar A=[B_{\mathrm{raw}}\mid C]$ with
$B_{\mathrm{raw}}\in\mathbb{R}^{n_X\times n_L}$ and
$C\in\mathbb{R}^{n_X\times n_X}$. We scale all latent-to-observed edges by
a nonnegative confounding strength $\lambda$, $B=\lambda B_{\mathrm{raw}}$,
leaving the observed-to-observed support unchanged. Observations are
\begin{equation}
x \;=\; (I_{n_X}-C)^{-1}\big(B\ell+\varepsilon^{\mathsf X}\big),
\end{equation}
i.e.\ latents and local observed shocks propagate through the directed
subgraph on sensors; larger $\lambda$ increases hidden common-cause leakage
into measurements. The default nonlinear generator applies two mixing
layers with leaky-ReLU or tanh nonlinearities and randomly drawn mixing
weights of controlled condition number.

\paragraph{Regimes and root causes.}
Normal data uses baseline noise scales drawn uniformly from
$[a_{\min}, a_{\max}]$, e.g.\ $[0, 8]$. For anomalous data we select
$\lfloor p/2 \rfloor$ of the $p$ observed nodes as root causes and redraw
their noise scales from a stronger range $[b_{\min}, b_{\max}]$, e.g.\
$[2, 12]$, leaving non-targets at their baseline scales. With $p = 6$ this
gives $3$ root causes. We draw $N_{\text{normal}} = N_{\text{anomalous}} =
5000$ i.i.d.\ samples from the same SCM family; other experiments use
different sample sizes, reported in their respective sections. Exact
hyperparameters are given in the scripts in our code repository.

\textbf{Exp 1: Nonlinear model with unobserved confounders:}
We use graphs with $p = 6$ observed and $m = 4$ latent variables (total $n=10$) and vary the confounding strength $\lambda \in \{3, 10\}$, which uniformly scales all latent-to-observed coefficients; a larger $\lambda$ makes the hidden confounders contribute more strongly to the observed variables. We take $\lfloor p/2 \rfloor = 3$ observed variables as ground-truth root causes. 

\textbf{Observation:}
Figure~\ref{fig:latent_str}(a) reports PRR for $\lambda \in \{3, 10\}$. \myalgo{} attains $99\%$ and $84\%$, ahead of every baseline in both settings: BARO ($91\%$, $65\%$), NSigma ($86\%$, $62\%$), RCG ($85\%$, $39\%$), RCD ($74\%$, $59\%$), and CIRCA ($62\%$, $38\%$). All methods degrade as confounding strengthens: \myalgo{} starting at $99\%$ and RCD at $74\%$, both drop by $15$ percentage points, while BARO, NSigma, CIRCA, and RCG drop by $26$, $24$, $24$, and $46$ points, respectively. This suggests that \myalgo{} is more robust to strong confounding and better preserves root-cause recovery than the baselines. 
The baselines fail because a non-root-cause descendant that shares a latent confounder with a root cause inherits a shift that persists even after conditioning on its observed parents, making it indistinguishable from a true root cause unless the latents are explicitly accounted for.

\textbf{Exp 2: Nonlinear model with heterogeneous anomalies:}
Baselines that perform RCA mainly by measuring marginal shifts implicitly assume a homogeneous anomaly, in which the true root causes exhibit the largest distributional shifts among all variables. This assumption often fails with multiple root causes: a shift originating at a root cause $R_1$ propagates to its descendants, so a downstream non-root-cause descendant can end up with a larger observed shift than another root cause $R_2$, which creates a failure case for these baselines.

We construct such heterogeneous anomalies by scaling the injected shift magnitude linearly with topological depth, so that shifts accumulate along directed paths and a downstream descendant out-shifts a genuine upstream root cause. In the generated data, this trap occurs (i.e., some non-root cause exhibits a larger marginal shift than some true root cause) in roughly $60\%$ of trials at $p=10$. We follow the same random graph generation process as in the previous experiment, but with no latent confounders and with $\lfloor p/2 \rfloor$ of the observed variables as root causes. We fix $5000$ normal and $5000$ anomalous samples and vary the number of observed variables $n \in \{5, 10\}$ ($n=p$ as $m=0$).

\textbf{Observation:}
Figure~\ref{fig:latent_str}(b) shows PRR as the system grows. \myalgo{} achieves the highest rate at both sizes, $98\%$ at $n=5$ and $54\%$ at $n=10$, outperforming all five baselines in both settings. The marginal-shift methods (BARO and NSigma) degrade the most, falling to $14\%$ at $n=10$; this is the failure mode the construction targets, since it makes ranking by shift magnitude wrong in most trials. The graph-based methods perform better (RCG $36\%$, CIRCA $34\%$, RCD $28\%$) because they propagate evidence along edges rather than scoring each node in isolation, so a descendant's large shift can be partly explained by its parents. However, they also degrade as the system grows, since their conditional tests involve larger conditioning sets and become less reliable at a fixed sample size. These results suggest that \myalgo{} is more robust to increasing system complexity, where the effects of multiple root causes can compound.

\textbf{Exp 3: Robustness to graph misspecification.}
We next examine the sensitivity of \myalgo{} to errors in the supplied graph. On the same $100$ datasets used for $\lambda = 10$ in \textbf{Exp 1}, we perturb \emph{only} the graph provided to \myalgo{}, so that any change in performance is attributable solely to the graph error. We compare three perturbations against the true-graph control (PRR of $88\%$ and $84\%$ in two independent runs): adding $50\%$ spurious directed and bidirected edges (PRR $= 84\%$), deleting $50\%$ of the bidirected edges (PRR $= 87\%$), and deleting $50\%$ of both directed and bidirected edges (PRR $= 76\%$). Under a paired McNemar test, neither edge addition ($p = 0.39$) nor bidirected-edge deletion ($p = 1.00$) changes accuracy significantly; only the perturbation that also deletes directed edges degrades it ($p = 0.019$), and even then \myalgo{} retains a PRR of $76\%$, above every baseline at $\lambda = 10$ (best: BARO, $65\%$).
These results are consistent with our analysis in Section~\ref{sec:graph-mis}, which covers misspecified graphs that can still represent the observed normal and anomalous distributions, i.e., $\{P_n, P_a\} \subseteq \mathcal{M}(G')$. Adding edges always preserves this condition, and in our instances deleting bidirected edges largely preserves it as well; in both cases, \myalgo{} performs on par with the control. Deleting directed edges, in contrast, typically takes $G'$ outside this class: each conditional $P(V \mid \mathrm{pa}(V))$ must be fit against an incomplete parent set, so the shift propagated along a missing edge cannot be explained away, and non-root-cause variables can appear as root causes.

\subsection{Sensitivity analysis of unequal sample sizes}
\label{sec:sample-sensitivity}
Algorithm~\ref{alg:dcm-rca} trains two DCMs, one on the normal data and one
on the anomalous data, each fit solely to its own sample. An implicit
unequal weighting between the two loss terms is therefore not the main risk
that sample-size imbalance introduces; what matters is whether each DCM
individually has enough data to learn its own SCM.

Suppose $r$ samples are sufficient to learn an SCM with neural networks.
Let $N_s$ and $N_a$ denote the normal and anomalous sample sizes, and
suppose we are testing whether variable $Y$ is a root cause. With equal
weighting, the DCM objective searches for a mechanism $f_Y$ consistent with
both the $N_s$ normal and the $N_a$ anomalous samples; if the loss does not
reach zero, no such mechanism exists, so $Y$ must be a root cause. However,
if $N_a \leq r$, the anomalous samples are insufficient to learn the SCM
reliably, which violates the universal-approximation assumption underlying
our test, and the algorithm may give incorrect results. We study this
empirically below.

\paragraph{Exp 6. Setup.}
We generate synthetic causal systems with $8$ observed variables and $4$
unobserved confounders. Graphs are moderately dense ($\sim\!80\%$ of
possible ancestral edges), with interventions on $\sim\!40\%$ of variables
($3$ root causes) at uniformly drawn strengths. Latent-to-observed edges
are moderately strong (scale $2$), and the maximum normal-regime noise
scale is $8$. Models are trained for $50$ epochs with equal weights on the
normal and anomalous losses, as in Algorithm~\ref{alg:dcm-rca}, ranking
candidates by Wasserstein discrepancy. We report the \emph{perfect recovery
rate} (PRR): an instance counts as a success only when the top-$k$
predictions exactly recover the $k$ true root causes (the same top-$k$ rule
is applied to BARO and RCD).

\begin{table}[h]
\centering
\caption{Perfect recovery rate (PRR) under varying normal ($N_s$) and
anomalous ($N_a$) sample sizes.}
\label{tab:sample-imbalance}
\begin{tabular}{rrcccr}
\toprule
$N_s$ & $N_a$ & DCM & BARO & RCD & \#Instances \\
\midrule
1000 & 1000 & \textbf{0.63} & 0.39 & 0.32 & 78 \\
1000 & 500  & \textbf{0.58} & 0.48 & 0.48 & 95 \\
500  & 500  & 0.47 & 0.47 & 0.31 & 51 \\
1000 & 100  & 0.27 & 0.35 & 0.24 & 131 \\
100  & 100  & 0.23 & 0.24 & 0.02 & 62 \\
\bottomrule
\end{tabular}
\end{table}

\paragraph{Findings.}
(i) Under moderate imbalance ($N_s = 1000$, $N_a = 500$), equal weighting
remains effective: DCM PRR stays close to the balanced setting ($0.58$ vs.\
$0.63$) and still matches or exceeds the baselines. (ii) When the anomalous
sample is much smaller ($N_a = 100$, $N_s = 1000$, a $10\!:\!1$ ratio), DCM
PRR drops sharply ($0.63 \rightarrow 0.27$) and BARO becomes competitive or
better. (iii) Balanced small samples ($100/100$) are hard for all methods;
scaling $N_s$ and $N_a$ together helps more than inflating $N_s$ alone.

\subsection{Causal Chamber (Real-World Physical Testbed)}
\label{app:chamber}

\textbf{Exp 4:} We evaluate \myalgo{} on the Causal Chambers benchmark~\citep{gamella2025causal}, a real-world physical testbed. We use the light-tunnel chamber in its standard configuration, which comes with a ground-truth DAG over $38$ variables and $57$ edges. Each experiment applies a single-target intervention of varying strength (strong, mid, or weak) to one of the tunnel's manipulable variables, shifting the target's distribution relative to the \texttt{uniform\_reference} experiment. We use the reference experiment as the normal dataset ($10$k samples) and each intervention experiment as an anomalous dataset ($1$k samples), and we drop variables that are constant within an experiment. This yields $52$ cases, which we evaluate in two conditions: with the true root cause \emph{observed}, and with it \emph{hidden}, so that it acts as a latent confounder among its former children. Its children then remain dependent through it, but it is no longer available to condition on. In the observed condition, $|R^*| = 1$; in the hidden condition, the hidden cause has between $1$ and $7$ children, all of which must be recovered as root causes.

\textbf{Observation:}
Figure~\ref{fig:latent_str}(c) reports PRR in both conditions. When the root cause is observed, there is no confounding and every case has a single root cause; here, \myalgo{} is competitive but does not lead ($88\%$, against $100\%$ for RCG and $90\%$ for RCD). Hiding the root cause changes the problem in two ways: it introduces a latent confounder, and it turns each of the hidden variable's children into a root cause, so a case may now have several root causes that must all be recovered. Under this harder setting, the ordering reverses: \myalgo{} attains $87\%$, against $73\%$ for the best baseline, and it is the only method whose performance barely changes between conditions ($-2$ points, against $-50$ for RCG and $-21$ for RCD). The gap is largest on the cases with multiple root causes: on these cases alone, \myalgo{} attains a PRR of $63\%$, while CIRCA reaches $25\%$, RCG $13\%$, and RCD $0\%$. RCD fails here by design: its elimination procedure keeps only one surviving candidate, so it can recover a single root cause but never a set of several. In contrast, \myalgo{} tests each variable separately for whether its own mechanism changed, so all children of the hidden cause are flagged and appear together at the top of the ranking.

\subsection{Microservice Datasets}
\label{app:micro}

\textbf{Exp 5:}
We evaluate \myalgo{} on SockShop, a cloud computing dataset recorded from a microservice-based replica of a web application specifically designed for evaluating root cause analysis methods~\citep{pham2024root}. The repository contains performance issues corresponding to five fault types: CPU hog (``cpu''), memory leak (``mem''), disk I/O stress (``disk''), network delay (``delay''), and packet loss (``loss''). Each fault type is injected five times into each of five services (carts, catalogue, orders, payment, and users), yielding $125$ anomaly datasets. Each dataset provides roughly $360$ normal and $360$ anomalous observations over $14$ recorded services, which we split into normal and anomalous segments based on the injection time. We assume that \emph{no} causal structure is available to \myalgo{}: instead of a call graph, we supply a \emph{confounded sink}, in which every recorded service points to the service under test and a single shared latent confounds all of them. The confounder is essential. A plain sink graph, with directed edges into the target but no edges among the remaining services, asserts that those services are mutually independent, which is false in these systems, where co-located services share load and infrastructure (in the normal regime, we measure a mean absolute pairwise correlation of $0.11$ on SockShop and $0.18$ on Online Boutique). Adding the shared latent removes this false independence claim while encoding no structural knowledge beyond ``any service may be responsible, and the services may share unobserved common causes.'' This places \myalgo{} at a disadvantage relative to two of the baselines: CIRCA and RCG are given the true edge-inverted call graph, while NSigma, BARO, and RCD do not use a graph at all.

\textbf{Exp 6:}
We repeat the evaluation on Online Boutique~\citep{pham2024root}, a second microservice benchmark with the same five fault types, each injected five times, again yielding $125$ anomaly datasets, but over $11$ services and with substantially longer traces (roughly $2100$ normal and $2100$ anomalous observations per case). The graph treatment is identical to Exp 5: \myalgo{} receives only the confounded sink, while CIRCA and RCG receive the true call graph. Because the two benchmarks differ in topology, service count, and trace length, together they test whether a method's behavior transfers across deployments rather than being tuned to one.

\textbf{Results:}
Table~\ref{tab:per-fault-topk} reports per-fault top-$k$ accuracy. On SockShop, \myalgo{} attains the highest average top-$1$ accuracy, $0.88$, ahead of NSigma ($0.75$), BARO ($0.74$), CIRCA ($0.65$), RCG ($0.54$), and RCD ($0.38$); every gap is significant under a paired test ($p<0.001$). It is best or tied for best on all five fault types, and its errors are near-misses rather than failures: its average top-$3$ accuracy reaches $0.99$. On Online Boutique, \myalgo{} again ranks first, at $0.78$, ahead of RCG ($0.71$), NSigma ($0.70$), BARO ($0.62$), CIRCA ($0.52$), and RCD ($0.49$), with an average top-$3$ accuracy of $0.90$; the margins over BARO, CIRCA, and RCD are significant, while those over NSigma and RCG are not at this sample size.

Two observations are worth highlighting. First, no baseline is consistently second: NSigma is the strongest competitor on SockShop, whereas RCG rises from fifth place on SockShop to second on Online Boutique and BARO drops from third to fourth. This is consistent with~\citet{ikram2025root}, who suggest that BARO may be tuned to SockShop, and with our synthetic and Causal Chamber results, where the ordering among baselines likewise changes with the setting. \myalgo{} is the only method that ranks first on both benchmarks. Second, packet loss is among the hardest fault types for \myalgo{} on both benchmarks ($0.76$ on SockShop, tied with disk, and $0.44$ on Online Boutique), whereas CPU and memory faults are recovered almost perfectly ($1.00$ and $0.96$--$1.00$). Loss perturbs a service only intermittently, so the induced mechanism shift is small relative to normal traffic variation. On Online Boutique, most baselines also score lowest on this category, which suggests that the difficulty lies largely in the anomaly signal itself rather than in the attribution method.

Notably, \myalgo{} attains these results without any causal structure, while the two graph-based baselines are given the true call graph. Together with the graph misspecification results in Section~\ref{exp:sim-dat}, this supports the view that the advantage of \myalgo{} does not depend on access to an accurate structure.

\definecolor{hdr}{RGB}{232,236,244}
\definecolor{avg}{RGB}{243,246,251}
\definecolor{ours}{RGB}{253,246,227}

\begin{table}[t]
\centering
\setlength{\tabcolsep}{3.5pt}
\renewcommand{\arraystretch}{1.2}
\caption{Per-fault Top-1/Top-3/Top-5 accuracy. Best value per column within each dataset is in bold.}
\label{tab:per-fault-topk}
\resizebox{\textwidth}{!}{%
\begin{tabular}{@{}ll ccc ccc ccc ccc ccc >{\columncolor{avg}}c >{\columncolor{avg}}c >{\columncolor{avg}}c@{}}
\toprule
\rowcolor{hdr}
& & \multicolumn{3}{c}{CPU} & \multicolumn{3}{c}{MEM} & \multicolumn{3}{c}{DISK}
& \multicolumn{3}{c}{DELAY} & \multicolumn{3}{c}{LOSS} & \multicolumn{3}{c}{\textsc{Average}} \\
\cmidrule(lr){3-5}\cmidrule(lr){6-8}\cmidrule(lr){9-11}\cmidrule(lr){12-14}\cmidrule(lr){15-17}\cmidrule(l){18-20}
Dataset & Method
& T1 & T3 & T5 & T1 & T3 & T5 & T1 & T3 & T5 & T1 & T3 & T5 & T1 & T3 & T5 & T1 & T3 & T5 \\
\midrule
\multirow{6}{*}{Sock Shop}
 & \cellcolor{ours}\textbf{DCM} & \cellcolor{ours}\textbf{1.00} & \cellcolor{ours}\textbf{1.00} & \cellcolor{ours}\textbf{1.00} & \cellcolor{ours}\textbf{0.96} & \cellcolor{ours}\textbf{0.96} & \cellcolor{ours}\textbf{0.96} & \cellcolor{ours}\textbf{0.76} & \cellcolor{ours}\textbf{1.00} & \cellcolor{ours}\textbf{1.00} & \cellcolor{ours}\textbf{0.92} & \cellcolor{ours}\textbf{1.00} & \cellcolor{ours}\textbf{1.00} & \cellcolor{ours}\textbf{0.76} & \cellcolor{ours}\textbf{1.00} & \cellcolor{ours}\textbf{1.00} & \cellcolor{ours}\textbf{0.88} & \cellcolor{ours}\textbf{0.99} & \cellcolor{ours}\textbf{0.99} \\
 & NSigma & \textbf{1.00} & \textbf{1.00} & \textbf{1.00} & \textbf{0.96} & \textbf{0.96} & \textbf{0.96} & 0.64 & 0.96 & \textbf{1.00} & 0.60 & 0.92 & \textbf{1.00} & 0.56 & \textbf{1.00} & \textbf{1.00} & 0.75 & 0.97 & \textbf{0.99} \\
 & BARO   & \textbf{1.00} & \textbf{1.00} & \textbf{1.00} & 0.48 & \textbf{0.96} & \textbf{0.96} & 0.68 & 0.96 & \textbf{1.00} & \textbf{0.92} & \textbf{1.00} & \textbf{1.00} & 0.64 & \textbf{1.00} & \textbf{1.00} & 0.74 & 0.98 & \textbf{0.99} \\
 & CIRCA  & 0.96 & \textbf{1.00} & \textbf{1.00} & 0.56 & \textbf{0.96} & \textbf{0.96} & 0.68 & 0.84 & \textbf{1.00} & 0.48 & \textbf{1.00} & \textbf{1.00} & 0.56 & \textbf{1.00} & \textbf{1.00} & 0.65 & 0.96 & \textbf{0.99} \\
 & RCD    & 0.76 & 0.92 & 0.96 & 0.28 & 0.48 & 0.48 & 0.32 & 0.52 & 0.52 & 0.28 & 0.48 & 0.52 & 0.28 & 0.32 & 0.44 & 0.38 & 0.54 & 0.58 \\
 & RCG    & 0.32 & 0.68 & 0.80 & 0.24 & 0.32 & 0.32 & 0.68 & 0.76 & 0.80 & 0.72 & 0.80 & 0.80 & 0.72 & 0.80 & 0.80 & 0.54 & 0.67 & 0.70 \\
\midrule
\multirow{6}{*}{Online Boutique}
 & \cellcolor{ours}\textbf{DCM} & \cellcolor{ours}\textbf{1.00} & \cellcolor{ours}\textbf{1.00} & \cellcolor{ours}\textbf{1.00} & \cellcolor{ours}\textbf{1.00} & \cellcolor{ours}\textbf{1.00} & \cellcolor{ours}\textbf{1.00} & \cellcolor{ours}0.64 & \cellcolor{ours}0.88 & \cellcolor{ours}\textbf{1.00} & \cellcolor{ours}\textbf{0.80} & \cellcolor{ours}\textbf{1.00} & \cellcolor{ours}\textbf{1.00} & \cellcolor{ours}\textbf{0.44} & \cellcolor{ours}0.60 & \cellcolor{ours}0.88 & \cellcolor{ours}\textbf{0.78} & \cellcolor{ours}0.90 & \cellcolor{ours}\textbf{0.98} \\
 & NSigma & \textbf{1.00} & \textbf{1.00} & \textbf{1.00} & \textbf{1.00} & \textbf{1.00} & \textbf{1.00} & 0.56 & \textbf{1.00} & \textbf{1.00} & 0.52 & 0.96 & \textbf{1.00} & \textbf{0.44} & 0.60 & 0.80 & 0.70 & \textbf{0.91} & 0.96 \\
 & BARO   & \textbf{1.00} & \textbf{1.00} & \textbf{1.00} & 0.36 & \textbf{1.00} & \textbf{1.00} & 0.64 & 0.88 & \textbf{1.00} & 0.64 & \textbf{1.00} & \textbf{1.00} & \textbf{0.44} & 0.64 & 0.72 & 0.62 & 0.90 & 0.94 \\
 & CIRCA  & 0.88 & \textbf{1.00} & \textbf{1.00} & 0.72 & \textbf{1.00} & \textbf{1.00} & 0.20 & 0.84 & \textbf{1.00} & 0.44 & 0.92 & \textbf{1.00} & 0.36 & 0.68 & \textbf{0.92} & 0.52 & 0.89 & \textbf{0.98} \\
 & RCD    & \textbf{1.00} & \textbf{1.00} & \textbf{1.00} & 0.96 & \textbf{1.00} & \textbf{1.00} & 0.20 & 0.36 & 0.40 & 0.20 & 0.28 & 0.28 & 0.08 & 0.32 & 0.40 & 0.49 & 0.59 & 0.62 \\
 & RCG    & 0.96 & \textbf{1.00} & \textbf{1.00} & 0.76 & 0.80 & 0.88 & \textbf{0.72} & 0.80 & 0.80 & \textbf{0.80} & 0.80 & 0.80 & 0.32 & \textbf{0.76} & 0.76 & 0.71 & 0.83 & 0.85 \\
\bottomrule
\end{tabular}%
}
\end{table}

\subsection{Computational Complexity}
\label{app:complexity}

\paragraph{Asymptotic complexity.}
DCM-RCA trains one DCM over the full causal graph to model the normal
distribution and then trains one candidate-specific DCM for each node in
the graph to test whether that node is a root cause. Let $|\mathbf{V}|$ be
the number of variables, $N$ the number of samples, $B$ the batch size,
$T$ the number of training epochs, and $c$ the cost of one training update
for a DCM over the full graph. Each epoch performs $\lceil N/B \rceil$
updates, so training one DCM costs $O(T \lceil N/B \rceil c)$, and
evaluating all candidate root causes sequentially requires training
$|\mathbf{V}|$ additional DCMs, giving total complexity
\[
O\!\left(|\mathbf{V}|\, T \lceil N/B \rceil c\right).
\]
However, the candidate-specific DCMs are independent once the normal DCM
is trained, so they can be executed in parallel. With $P$ parallel workers,
the wall-clock complexity becomes approximately
\[
O\!\left(T \lceil N/B \rceil c
\left(1+\left\lceil \frac{|\mathbf{V}|}{P}\right\rceil\right)\right),
\]
up to parallelization overhead and memory constraints. Thus, although the
total computational cost grows linearly with the number of candidate
variables, the practical runtime can be substantially reduced through
parallel execution: for $T=50$, $N=5000$, $B=256$, $|\mathbf{V}|=10$ and
$P=15$, the sequential cost is $11{,}000c$ versus $2000c$ in parallel, a
$5.5\times$ speedup.

\paragraph{Empirical wall-clock time.}
Table~\ref{tab:runtime} reports per-dataset runtime for DCM
(\texttt{dcm\_flow}, $50$ epochs) against BARO and RCD in the same
synthetic setting ($12$ variables, $4$ latents, $N_s = N_a = 1000$).

\begin{table}[h]
\centering
\caption{Per-dataset wall-clock time in the synthetic setting
($12$ variables, $4$ latents, $N_s = N_a = 1000$).}
\label{tab:runtime}
\begin{tabular}{ll}
\toprule
Method & Typical time / dataset \\
\midrule
BARO        & $\sim 0.01$\,s \\
RCD         & $\sim 0.4$\,s (median) \\
DCM (ours)  & $\sim 41$\,s (median) \\
\bottomrule
\end{tabular}
\end{table}

DCM is slower than RCD by $10^{2}\times$ and than BARO by
$10^{3}$--$10^{4}\times$, as expected: BARO is a lightweight
distributional score, RCD a localized CI-test search, whereas DCM trains a
full multi-node, multi-epoch causal model (here with up to $10$ parallel
workers). DCM runtime scales with sample size ($\sim 9$\,s at $N=100$,
$\sim 41$\,s at $N=1000$, and $\sim 2$--$3$\,min at $N_s=2000$,
$N_a=1000$), while BARO and RCD remain sub-second throughout. These
baselines, however, cannot reliably handle systems with unobserved
confounders, which our algorithm can.

We do not give a matching closed-form complexity for BARO, RCD, and CIRCA:
BARO performs a single pass per variable; RCD's own analysis notes that
CI-test-based discovery is in general exponential in the number of nodes
for non-sparse graphs (its localized search targets avoiding this in
practice, without a published closed form); and CIRCA's complexity is
dominated by its causal-discovery step. We therefore view this as an
\emph{accuracy--compute trade-off}: DCM-RCA substantially improves accuracy
under latent confounding at the cost of
runtime, while the baselines remain preferable when raw speed is the
priority. DCM-RCA's cost is justified whenever reliability under unobserved
confounding matters.

\begin{algorithm}[t]
\caption{$\operatorname{RunDCM}(G, \mathbb{G})$}
\label{alg:run-nf}
\begin{algorithmic}[1]
\STATE \textbf{Input:} Causal graph $G = (\mathcal{V}, \mathcal{E})$, DCM $\mathbb{G}$.
\STATE $\hat{\mathbf{v}} \leftarrow \emptyset$ \hfill \COMMENT{Initialize generated samples}
\STATE $conf \leftarrow \emptyset$ \hfill \COMMENT{Initialize confounding noise storage}

\FOR{$U \in latent\_confounders(G)$}
    \STATE $v_1, v_2 = U.ch1, U.ch2$ \hfill \COMMENT{Get the two observed children of latent confounder $U$}
    \STATE $z_U \sim p(z)$ \hfill \COMMENT{Sample one shared latent noise}
    \STATE $conf[v_1] \leftarrow Append(conf[v_1], z_U)$ \hfill \COMMENT{Assign shared noise to $v_1$}
    \STATE $conf[v_2] \leftarrow Append(conf[v_2], z_U)$ \hfill \COMMENT{Assign shared noise to $v_2$}
\ENDFOR

\FOR{$V_i \in \mathcal{V}$ in causal graph $G$ topological order}
    \STATE $par = get\_parents(V_i, G)$ \hfill \COMMENT{Get observed parents of $V_i$}
    \STATE $exos \sim p(z)$ \hfill \COMMENT{Sample exogenous noise for $V_i$}
    \STATE $conf_i = conf[V_i]$ \hfill \COMMENT{Collect latent confounding noise for $V_i$}
    \STATE $v_i = \mathbb{G}_{\theta_i}(exos, conf_i, \hat{\mathbf{v}}_{par})$ \hfill \COMMENT{Generate $V_i$ using its Causal Normalizing Flow module}
    \STATE $\hat{\mathbf{v}} \leftarrow Append(\hat{\mathbf{v}}, v_i)$ \hfill \COMMENT{Store generated value}
\ENDFOR

\STATE \textbf{Return} Samples $\hat{\mathbf{v}}$ or Fail
\end{algorithmic}
\end{algorithm}

\begin{figure}[t]
    \centering
    \includegraphics[width=0.5\linewidth]{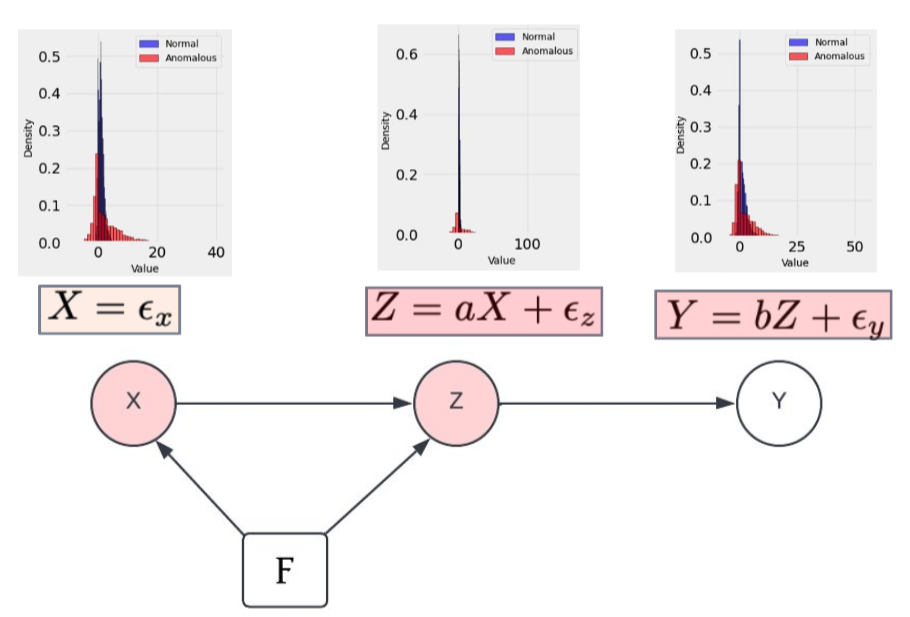}
    \caption{Non-rc $Y$ has higher shift in median than RC  $X$.}
    \label{fig:hetero-motiv}
\end{figure}

\section{Reproducibility Statement}
The full set of assumptions required for our soundness, completeness, and graph misspecification results is stated in Appendix~\ref{sec:assumptions}, and the corresponding proofs are provided in Appendices~\ref{sec:proof-solrc}--\ref{sec:proof-mis}. Section 5 describes the datasets, causal graphs, anomaly settings, evaluation metrics, baselines, and experimental protocols needed to reproduce the main results. Training details, implementation choices, and dataset-specific settings are given in Appendix~\ref{appex:exp}, with further details for nonlinear models with unobserved confounders in Appendix~\ref{app:sim-dat} and computational complexity in Appendix~\ref{app:complexity}. Since DCM-RCA trains one DCM for the full graph and one candidate-specific DCM for each node, the candidate-specific models can be trained in parallel, reducing practical wall-clock time. We report variability information where applicable, including error bars for the simulated experiments and accuracy across multiple anomaly instances for the real-world datasets; the reported perfect recovery and top-1 accuracy results are computed over repeated runs or multiple injected anomaly datasets. The datasets are either synthetically generated using the procedures described in the paper or publicly available. Code, with instructions for reproducing the reported experiments, is available at \url{https://github.com/Musfiqshohan/RCA-DCM}.

\clearpage

\end{document}